\documentclass{article}[10pt]

\usepackage[T1]{fontenc}
\usepackage[dvipsnames]{xcolor}
\usepackage{graphicx}
\usepackage{todonotes}
\usepackage{preamble}
\usepackage{my_theorems}
\usepackage{pifont}
\usepackage{latexsym}
\usepackage{hyperref}

\usepackage{select_version}
\usepackage{my_macro}
\usepackage{laetitia_macro}
\usepackage{etienne_macros}

\begin{document}
\title{Weakly synchronous systems with three machines are Turing powerful}
\author{Cinzia Di Giusto \and Davide Ferré \and  Etienne Lozes \and Nicolas Nisse}%
\date{}
\maketitle              % typeset the header of the contribution

\begin{abstract}
    Communicating finite-state machines (CFMs) are a Turing powerful model of
    asynchronous message-passing distributed systems. In weakly synchronous systems, processes communicate through phases in which messages are first sent and then received, for each process. Such systems enjoy  a limited
    form of synchronization, and for some communication models, this restriction is enough 
    to make the reachability problem decidable. In particular, we explore
    the intriguing case of p2p (FIFO) communication, for which the reachability problem
    is known to be undecidable for four processes, but decidable for two. We show that the configuration reachability problem for weakly synchronous systems of three processes is undecidable.
    This result is heavily inspired by our study on the treewidth of the Message Sequence Charts (MSCs) that might be generated by such systems. In this sense, the main contribution of this work is a weakly synchronous system with three processes that generates MSCs of arbitrarily large treewidth.

\end{abstract}

\section{Introduction}
% !TEX root = ../conference.tex

Systems of communicating finite-state machines (CFMs) are a simple, yet expressive, model of
asynchronous message-passing distributed systems. In this model, each 
machine performs a sequence of  send and receive actions, where a send action can be matched by a receive
action of another machine. For instance, the system in Fig.~\ref{fig:system-exampleNew} (left),
%%
%\begin{figure} [t]
%\centering
%\captionbox{Example of a system of CFMs.\label{fig:system-example}}
%[.6\textwidth]{\include{fig/fig-system-example}}%
%% 
%\captionbox{Example of an MSC~\label{fig:msc-example}}
%[.35\textwidth]{\include{fig/fig-msc-example}}%
%%
%\end{figure}
%%
models a protocol  between three processes $a$, $b$, and $r$.

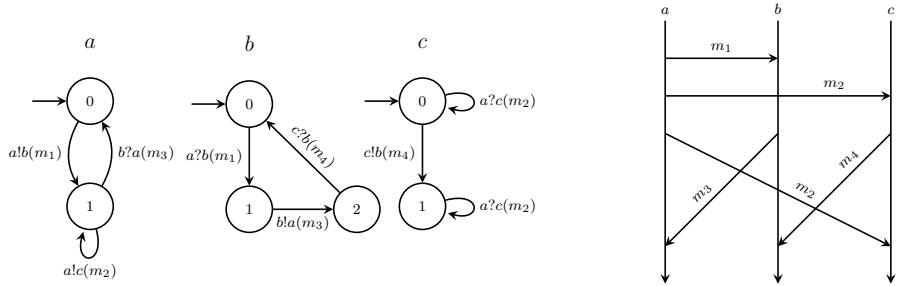
\begin{figure} [t]
\centering
    \begin{subfigure}[b]{0.7\textwidth}
         \raggedright
         \begin{tikzpicture}[every state/.style={text=black, scale =1}, semithick,
        font={\fontsize{8pt}{12}\selectfont},]
    \node (s) {\large $a$};
    \node[state, initial, below=0.3cm of s] (s0) {$0$};
    \node[state, below of=s0] (s1) {$1$};
    
    \node[right=1.8cm of s] (c) {\large $b$};
    \node[state, initial, below=0.3cm of c] (c0) {$0$};
    \node[state, below of=c0] (c1) {$1$};	
    \node[state, right=0.8cm of c1] (c2) {$2$};

    \node[right=2cm of c] (d) {\large $c$};
    \node[state, initial, below=0.3cm of d] (d0) {0};
    \node[state, below of=d0] (d1) {$1$};
    
    \tikzset{->,}
    \draw
    (s0) edge[bend right, left] node{$\send{a}{b}{m_1}$} (s1)
    (s1) edge[loop below] node{$\send{a}{c}{m_2}$} (s1)
    (s1) edge[bend right, right] node{$\rec{b}{a}{m_3}$} (s0)
    
    (c0) edge[left] node{$\rec{a}{b}{m_1}$} (c1)
    (c1) edge[below] node{$\send{b}{a}{m_3}$} (c2)
    (c2) edge[above] node [sloped] {$\rec{c}{b}{m_4}$} (c0)
    
    (d0) edge[loop right, right] node{$\rec{a}{c}{m_2}$} (d0)
    (d1) edge[loop right, right] node{$\rec{a}{c}{m_2}$} (d1)
    (d0) edge[left] node{$\send{c}{b}{m_4}$} (d1);
    
\end{tikzpicture}    
        % \caption{Example of a system of CFMs.}
         \label{fig:system-example}
     \end{subfigure}
     \begin{subfigure}[b]{0.28\textwidth}
         \centering
         \begin{tikzpicture}[every state/.style={text=black, scale =0.75}, semithick,
        font={\fontsize{8pt}{12}\selectfont}]
	\draw[->] (0,0) [above] node {$a$} -- (0, -3.5);
	\draw[->] (1.5,0) [above] node {$b$} -- (1.5, -3.5);
	\draw[->] (3,0) [above] node {$c$} -- (3, -3.5);
	\draw[->] (0, -0.5) -- node [above] {$m_1$} (1.5, -0.5);
	\draw[->] (0, -1) -- node [above, pos=0.75] {$m_2$} (3, -1);
	\draw[->] (0, -1.5) -- node [above, sloped, pos=0.6] {$m_2$} (3, -3);
	\draw[->] (1.5, -1.5) -- node [above, sloped, pos=0.6] {$m_3$} (0, -3);
	\draw[->] (3, -1.5) -- node [above, sloped, pos=0.3] {$m_4$} (1.5, -3);
\end{tikzpicture}
         %\caption{Example of an MSC.}
         \label{fig:msc-example}
     \end{subfigure}
     \caption{Example of a system of $3$ CFMs (left) and of an MSC generated by it (right). $a!b(m_1)$ (resp., $b?a(m_1)$) denotes the sending (reception) of message $m_1$ from (by) process $a$ to (from) process $b$.}
%\captionbox{Example of a system of CFMs.\label{fig:system-example}}
%[.6\textwidth]{\include{fig/fig-system-example}}%
% 
%\captionbox{Example of an MSC~\label{fig:msc-example}}
%[.35\textwidth]{\include{fig/fig-msc-example}}%
%%
\label{fig:system-exampleNew}
\end{figure}
%

%$p!q(a)$ (resp., $p?q(a)$) denotes the sending (reception) of message $a$ from (by) process $a$ to (from) process $q$

A computation of such a system can be represented
graphically by a Message Sequence Chart (MSC), a simplified version of the 
ITU recommendation~\cite{messagesequencecharts}. Each machine of the system
has its own ``timeline'' on the MSC, where  actions are listed in the order
in which they are  executed, and message arrows link a send action to
its matching receive action. For instance, the MSC of Fig.~\ref{fig:system-exampleNew} (right)
represents one of the many computations of the system in Fig.~\ref{fig:system-exampleNew} (left). 
The set of all MSCs
that the system may generate is determined both by the machines, since 
the sequence of actions of each timeline must be a sequence of action in the corresponding
CFM, and by the ``transport layer'' or ``communication model'' employed by the machines. 
Roughly
speaking, a communication model is  a class of MSCs that are considered
``realizable'' within that model of communications. For instance, 
for rendezvous synchronization,
an MSC is considered to be realizable with synchronous communication 
if the only path between a sending and its matching receipt is
the direct one through the message arrow that relates them. Among the various
communication models that have been considered, we can cite  p2p (or FIFO) model,
where each ordered pair of machines defines a dedicated FIFO queue; causal ordering (CO),
where a message cannot overtake the messages that were sent causally before it;
the mailbox model, where each machine holds a unique FIFO queue for all
incoming messages; the bag (or simply asynchronous) model, where a message 
can overtake any other message (see~\cite{DBLP:journals/tcs/BasuB16,DBLP:journals/dc/Charron-BostMT96,DBLP:journals/pacmpl/GiustoFLL23} 
for various presentations of these communication models).

The configuration reachability problem for a system of CFMs consists in checking whether a control state, together with a given content of the queues, is reachable from the initial state. This problem is decidable for synchronous communication, as the state space of the system is
finite, and also for bag communication, by reduction to Petri nets~\cite{Mayr1981AnAF}. For
other communication models, as soon as two machines are allowed to exchange messages through
two FIFO queues, reachability becomes undecidable~\cite{DBLP:journals/jacm/BrandZ83}. Due to this strong limitation,
there has been a wealth of work that tried to recover decidability of  reachability by considering
systems of CFMs that are ``almost synchronous''.

In weakly synchronous systems, processes communicate through phases in which messages are first sent and then received, for each process; graphically, the MSCs of such systems
are the concatenation of smaller, independent MSCs, within which no send happens after a receive.
For instance, the MSC in Fig.~\ref{fig:system-exampleNew} (right) is weakly synchronous, as it is
the concatenation of three ``blocks'' (namely $\{m_1\}$, $\{m_2\}$, and $\{m_2,m_3,m_4\}$), within which all sends of a given machine 
happen before all the receives of this same machine.
It is known that  reachability  is decidable for mailbox weakly synchronous 
systems~\cite{BolligGFLLS21-long}, whereas it is undecidable for either p2p or CO weakly synchronous systems
with at least four machines. On the other hand,  reachability  is decidable for two machines (since any p2p MSC with 
two machines  is  also mailbox). In this work, 
we deal with weakly synchronous systems with three machines, and conclude that
 reachability  is undecidable for these systems.
Our result is based on a study of the unboundedness of the treewidth for MSCs that may be generated by these systems.
The first contribution of this work is a weakly synchronous system with 
only three machines that is ``treewidth universal'', in the sense that it
may generate MSCs of arbitrarily large treewidth. The second
contribution, strongly inspired by 
the  treewidth universal system, is showing that weakly synchronous systems with three processes are Turing powerful. 
To do so, we establish a one-to-one correspondence between the
computations of a FIFO automaton (a finite state machine that may push and
pop from a FIFO queue, which is known to be a Turing powerful computational
model) on the one hand, and a subset of the MSCs of the treewidth universal
system on the other hand.

%To do so, we  establish a 
%correspondence between 
%computations of a FIFO automaton (a finite state machine that may push and pop from a FIFO queue, which
%is known to be a Turing powerful computational model), and computations
%of a weakly synchronous system with three processes, which is built in a way that it ``simulates'' the FIFO automaton.
\vspace{0.5em}
\noindent{\bf Related work.}
Beyond weakly synchronous systems, several similar notions have been considered to try to capture the intuition of an ``almost synchronous'' system. Reachability  of existentially bounded
systems~\cite{DBLP:conf/fossacs/LohreyM02,GKM07} is decidable for FIFO, CO, p2p, or bag
communications. Synchronizable systems~\cite{DBLP:conf/www/BasuB11} were an attempt to
define a class of systems with good decidability properties, however  reachability 
for such systems with FIFO communications is undecidable~\cite{DBLP:conf/icalp/FinkelL17}. The status of
 reachability for
$k$-stable systems~\cite{DBLP:journals/fmsd/AkrounS18} is unknown. Finally,  reachability 
for 
$k$-synchronous systems~\cite{DBLP:conf/cav/BouajjaniEJQ18} is decidable
for FIFO, CO, p2p, or bag communications.

Another form of under-approximation of the full behaviour of a system of
CFMs is the bounded context-switch reachability problem,
which is known to be decidable for systems of CFMs, even with a controlled form of
function call~\cite{HeussnerLMS10,DBLP:conf/tacas/TorreMP08}.

%{\color{blue}N:Davide and I do not understand this sentence, and it is strange that it is here and not in the conclusion.}
%We conjecture that the reachability problem for weakly synchronous systems with bag communications is
%decidable in NP by reduction to reversal-bounded counter machines~\cite{reversal-bounded-revisited,ibarra1978reversal}.
Finally, weak synchronisability share some similarities with reversal-bounded
counter machines~\cite{reversal-bounded-revisited,ibarra1978reversal}: in the context
of bag communications, a send is a counter increment, a receive a decrement, and weak
synchronisability is a form of bounding the number of reversals of increment
and decrement phases.

\vspace{0.5em}
\noindent{\bf Outline.} Section~\ref{sec:prelim} introduces the necessary terminology. 
Section~\ref{sec:tw_ws_pp} presents the weakly synchronous system with three machines
that may generate MSCs of arbitrarily large treewidth. Then,
Section~\ref{sec:reach} discusses the undecidability of the configuration reachability problem
for weakly-synchronous systems with three machines. 
Finally, Section~\ref{sec:concl} concludes with some final remarks. The Appendix contains  proofs and additional material.

\section{MSCs and communicating automata}\label{sec:prelim}
% !TEX root = ../conference.tex

We recall  here  concepts and definitions related to MSCs and communicating automata. 
Assume a finite set of processes $\Procs$ and a finite set of messages $\Msg$.
%The set of \pp channels is $\Ch = \{(p,q) \in \Procs \times \Procs \mid p \neq q\}$.
%
A send action is of the form $\sact{p}{q}{\msg}$
where $p,q \in \Procs$ and $\msg \in \Msg$; it is executed by $p$ and sends message $\msg$ to process $q$.
The corresponding receive action, executed by $q$, is
$\ract{p}{q}{\msg}$.
%
%For $(p,q) \in \Ch$, 
Let
$\pqsAct{p}{q} = \{\sact{p}{q}{\msg} \mid \msg \in \Msg\}$ and
$\pqrAct{p}{q} = \{\ract{p}{q}{\msg} \mid \msg \in \Msg\}$.
For $p \in \Procs$, we set
$\psAct{p} = \{\sact{p}{q}{\msg} \mid q \in \Procs \setminus \{p\}$ and $\msg \in \Msg\}$, etc.
Moreover, $\pAct{p} = \psAct{p} \cup \qrAct{p} \cup \{ \varepsilon \}$ will denote the set of all actions that are
executed by $p$.
Finally, $\Act = \bigcup_{p \in \Procs} \pAct{p}$
is the set of all the actions.

\begin{definition}[\pp \MSCs]\label{def:msc}
A \emph{(\pp) \MSCs} $\msc$  over $\Procs$ and $\Msg$ is a tuple $\msc = (\Events,\procrel,\lhd,\lambda)$
where $\Events$ is a finite (possibly empty) set of \emph{events}
and $\lambda: \Events \to \Act$ is a labeling function.
For $p \in \Procs$, let $\Events_p = \{e \in \Events \mid \lambda(e) \in \pAct{p}\}$ be the set of events
that are executed by $p$.
$\procrel$ (the \emph{process relation}) is the disjoint union $\bigcup_{p \in \Procs} \procrel_p$
of relations ${\procrel_p} \subseteq \Events_p \times \Events_p$ such that
$\procrel_p$ is the direct successor relation of a total order on $\Events_p$.
% For an event $e \in \Events$, a set of actions $A \subseteq \Act$, and a relation 
% $\rel \subseteq \Events \times \Events$,
% let $\sametype{e}{A}{\rel} = |\{f \in \Events \mid (f,e) \in \rel$ and $\lambda(f) \in A\}|$.
${\lhd} \subseteq \Events \times \Events$ (the \emph{message relation}) satisfies the following:
\begin{itemize}\itemsep=0.5ex
\item[(1)] for every pair $(s,r) \in {\lhd}$, there is a send action $\sact{p}{q}{\msg} \in \Act$ such that
$\lambda(s) = \sact{p}{q}{\msg}$, $\lambda(r) = \ract{p}{q}{\msg}$, and $p\neq q$;
\item[(2)] for all $r \in \Events$ with $\lambda(r)=\ract{p}{q}{\msg}$, there is a unique $s \in \Events$ such that $s \lhd r$;
\item[(3)] letting ${\le}_\msc = ({\procrel} \cup {\lhd})^\ast$,
we require that $\le_\msc$ is a partial order;
\item[(4)] for every $s_1 \in \Events$ and pair $(s_2,r_2) \in {\lhd}$
with $\lambda(s_1)=\sact{p}{q}{\msg_1}$ and $\lambda(s_2)=\sact{p}{q}{\msg_2}$,
if
$s_1\procrel_p^{+} s_2$, 
then there exists $r_1$ such that $(s_1,r_1)\in \lhd$ and $r_1\procrel_q^{+} r_2$.
\end{itemize}
\end{definition}

Condition (1) above ensures that message arrows relate a send
event to a receive event on a distinct machine.
By Condition (2), every receive event has a matching send event.
Note that, however, there may be unmatched send events in an MSC. An MSC
is called \emph{orphan free} if all send events are matched. 
Condition (3) ensures that there exists at least one scheduling of all events
such that each receive event happens after its matching send event.
Condition (4) captures the p2p communication model: a message cannot overtake another
message that has the same sender and same receiver as itself.

Let $\msc = (\Events,\procrel,\lhd,\lambda)$ be an MSC, then 
$\SendEv{\msc} = \{e \in \Events \mid \lambda(e)$ is a send
action$\}$,
$\RecEv{\msc} = \{e \in \Events \mid \lambda(e)$ is a receive
action$\}$,
$\Matched{\msc} = \{e \in \Events \mid$ there is $f \in \Events$
such that $e \lhd f\}$, and
$\Unm{\msc} = \{e \in \Events \mid \lambda(e)$ is a send
action and there is no $f \in \Events$ such that $e \lhd f\}$.
We do not distinguish isomorphic MSCs. % and
%let $\ppMSCs$ be the set of all \MSCss over the given sets $\Procs$ and $\Msg$.
Let $E \subseteq \Events$ such that $E$ is ${\le_\msc}$-\emph{downward-closed}, i.e,
for all $(e,f) \in {\le_\msc}$ such that $f \in E$, we also have $e \in E$.
Then the MSC $\msc' = (E,\procrel,\lhd,\lambda)$ obtained by restriction to $E$
is called a \emph{prefix} of $\msc$.
If $\msc_1 = (\Events_1,\procrel_1,\lhd_1,\lambda_1)$ and
$\msc_2 = (\Events_2,\procrel_2,\lhd_2,\lambda_2)$ are two MSCs,
their \emph{concatenation} $\msc_1 \cdot \msc_2 = (\Events,\procrel,\lhd,\lambda)$ is 
as expected: $\Events$ is the disjoint union of $\Events_1$ and $\Events_2$,
${\lhd}  = {\lhd_1} \cup {\lhd_2}$, $\lambda$ is the ``union'' of $\lambda_1$
and $\lambda_2$, and ${\procrel} = {\procrel_1} \cup {\procrel_2} \cup R$.
Here, $R$ contains, for all $p \in \Procs$ such that $(\Events_1)_p$ and
$(\Events_2)_p$ are non-empty, the pair $(e_1,e_2)$, where $e_1$ and $e_2$ are the last and the first event executed by $p$ in $M_1$ and $M_2$, respectively.
Due to condition (4), concatenation is a partially defined operation:  $\msc_1\cdot\msc_2$ is defined if  
for all 
$s_1 \in \Unm{\msc_1}$ and
$s_2 \in \SendEv{\msc_2}$ that have the same sender and destination 
($\lambda(s_1) \in \pqsAct{p}{q}$ and $\lambda(s_2) \in \pqsAct{p}{q}$),
we have $s_2 \in \Unm{\msc_2}$. In particular, $\msc_1\cdot\msc_2$ is defined
when $\msc_1$ is orphan-free.
Concatenation is associative.

We recall from \cite{BolligGFLLS21} the definition of   weakly synchronous MSC. We say that an MSC is weakly synchronous if it can be broken into phases where  all sends are scheduled before all receives. 

%\cinzia{do we really need prefix and concatenation?}

\begin{definition}[weakly synchronous]\label{def:weaksync-new}
We say that $\msc \in \MSCs$ is
\emph{weakly synchronous} if it is of the form
$\msc = \msc_1 \cdot \msc_2 \cdots \msc_n$
such that for every $\msc_i = (\Events,\procrel,\lhd,\lambda)$ 
$\SendEv{\msc_{i}}$ is
a ${\le_{\msc_i}}$-downward-closed set.
\end{definition}

We now recall the definition of communicating system, which consists of finite-state machines $A_p$
(one per process $p \in \Procs$) that can exchange messages.

\begin{definition}[communicating system]\label{def:cs}
A \emph{(communicating) system} over $\Procs$ and $\Msg$ is a tuple
   $ \Sys = ((A_p)_{p\in\procSet})$. For each
   $p \in \Procs$, $A_p = (Loc_p, \delta_p, \ell^0_p, \ell^{acc}_p)$ is a finite transition system where:
$\Loc_p$ is the finite set of (local) states of $p$,
 $\delta_p    \subseteq \Loc_p \times \pAct{p} \times \Loc_p$  (also denoted $\ell \xrightarrow[A_p]{a} \ell'$) is the
    transition relation of $p$,  
 $\ell^{acc}_p \in \Loc_p $ is the  final state of $p$.
\end{definition}

Given $p \in \Procs$ and a transition $t = (\ell,a,\ell') \in \delta_p$, we let
$\tsource(t) = \ell$, $\ttarget(t) = \ell'$, $\tlabel(t) = a$, and
$\tmessage(t) = \msg$ if $a \in \msAct{\msg} \cup \mrAct{\msg}$.

%We will define
%the language of $\Sys$ directly as a set of MSCs. 
%
%Let $\MSCs = (\Events,\procrel,\lhd,\lambda)$ be an MSC.
An \emph{accepting run} of $\Sys$ on  an \MSCs $\msc$ is a mapping
$\rho: \Events \to \bigcup_{p \in \Procs} \delta_p$
that assigns to every event $e$ the transition $\rho(e)$
that is executed at $e$ by $A_p$. Thus, we require that
\begin{enumerate*}[label={(\roman*)}]
\item for all $e \in \Events$, we have $\tlabel(\rho(e)) = \lambda(e)$,
\item for all $(e,f) \in {\procrel}$, $\ttarget(\rho(e)) \xrightarrow[A_p]{\varepsilon}^* \tsource(\rho(f))$,
\item for all $(e,f) \in {\lhd}$, $\tmessage(\rho(e)) = \tmessage(\rho(f))$,
\item for all $p \in \Procs$ and $e \in \Events_p$ such that there is no $f \in \Events$ with $f \procrel e$, we have $\tsource(\rho(e)) = \ell_p^0$,
\item for all $p \in \Procs$ and $e \in \Events_p$ such that there is no $f \in \Events$ with $e \procrel f$, we have $\ttarget(\rho(e)) = \ell^{acc}_p$ and,
\item  $\Unm{\msc} = \emptyset$.
\end{enumerate*}
Essentially, in an accepting run of $\Sys$ every $A_p$ takes a sequence of transitions that lead to its final state $\ell^{acc}_p$, and such that each send action will have a matching receive action (i.e., there are no unmatched messages).
 The \emph{language} of $\Sys$ is $\Lang{\Sys} = \{\msc \in \MSCs \mid$ there is an accepting run of $\Sys$ on $\msc\}$. 
 We say that $\Sys$ is weakly synchronous if for all $\msc\in\Lang{\Sys}$, 
 $\msc$ is weakly synchronous.

The \emph{emptiness problem} is the decision problem that takes as 
input a system $\Sys$ and addresses the question ``is $\Lang{\Sys}$ empty?''. 
This problem is a configuration reachability problem, and under several circumstances,
its decidability is closely related to the one of the 
control state reachability problem. In this work, we will study the emptiness problem
with the additional hypothesis that $\Sys$ is a weakly synchronous system with three machines only.

Finally, we recall the less known notion of ``FIFO automaton'', a finite
state machine that can push into and pop from a FIFO queue. This is 
a system of communicating machines with just one machine, whose semantics is a set
of MSCs with a single timeline, for which we exceptionally relax condition (1) of Definition~\ref{def:msc}, 
so to allow a send event and its matching receive event to occur on the same machine.
The following result is proved in \cite[Lemma~4]{DBLP:conf/icalp/FinkelL17}.
\begin{lemma}[\cite{DBLP:conf/icalp/FinkelL17}]
   \label{lem:FIFO-automaton-emptiness-is-not-decidable}
   The emptiness problem for FIFO automata is undecidable.
\end{lemma}

%\section{Treewidth and minors}
%\input{chapters/ch-prelim-graphs.tex}

\section{Treewidth of weakly synchronous \pp MSCs}
\label{sec:tw_ws_pp}
% !TEX root = ../conference.tex

%Weakly synchronous MSCs can be expressed as a concatenation of \emph{exchanges}, where an exchange is an MSC in which all the send events can be scheduled before all the receive events. %Fig.\ref{fig:weakly_sync_example} shows a simple graphical characterization of an exchange. 
%We can further restrict the class of weakly synchronous MSCs by considering a specific communication model, such as \pp or \mb. 
There is a strong correlation between MSCs and graphs. An MSC is a directed graph (digraph in the following) where the nodes are the events of the MSC and the arcs are represented by the $\procrel$ and the $\lhd$ relations. We are, therefore, able to use some tools and techniques from graph theory to possibly derive some interesting results about MSCs. A graph parameter which is particularly important in this context is the \emph{treewidth}~\cite{bodlaender:partialK} mostly due to Courcelle's theorem that, roughly, states that many properties can be checked in classes of MSCs with bounded treewidth\footnote{Since we do not explicitly use  tree-decompositions, we refer to~\cite{bodlaender:partialK} for their formal definitions. Definitions are recalled in Appendix \ref{app:defs} for the reviewer convenience.}. For instance, in \cite{BolligGFLLS21}, it is shown that the class of weakly synchronous mailbox MSCs has bounded treewidth. Interestingly enough, it is also shown that the bigger class of weakly synchronous \pp MSCs has unbounded treewidth, by means of a reduction to the  Post correspondence problem. 
Here we give a more direct proof, for all weakly synchronous systems that have at least three processes. We begin with some terminology:

\begin{definition}
\emph{To  contract} an arc $(u,v)$ in a (di)graph $G$ means replacing $u$ and $v$ by a single vertex whose neighborhood is the union of the neighborhoods of $u$ and $v$.
A (di)graph $H$ is a \emph{minor} of a (di)graph $G$ if $H$ can be obtained from a subgraph of $G$ by contracting some edges/arcs. 
\end{definition}

%The proof is a bit sketchy, so we came up with an alternative one; this is what this section is about. We also investigate how many processes we need for the treewidth of weakly synchronous \pp MSCs to become unbounded. Since the \pp and \mb communication models are equivalent if we only have two processes, we can trivially observe that the treewidth of weakly synchronous \pp MSCs is bounded for $p=2$, where $p$ is the number of processes. We show that it is not the case when $p=3$. 

% \begin{figure}[h]
%     \centering
%     \includegraphics[width=\textwidth]{ws_p2p_minor_k3_3}
%     \caption{On the top left, an example of graph that has $K_3$ as a minor. In step 1 we do some edge contractions to connect the first two red nodes (edges in red). In step 2 we do the same for all the remaining colored nodes, and the graph we obtain is the same as that of Fig.\ref{fig:ex1_minor_k3}, which has $K_3$ as a minor.} 
%     \label{fig:ws_p2p_minor_k3_3}
% \end{figure}

Next, we show how to build a family of weakly synchronous MSCs with three processes ($a$, $b$ and $c$) and unbounded treewidth. We want to find a class of MSCs that admit grids of unbounded size as a minor.
The idea is illustrated in Fig.~\ref{fig:ws_p2p_tw_3}, and it consists in bouncing groups of messages between processes so to obtain the depicted shape. The class of MSCs is indexed by two non-zero natural numbers: $h$ and $\ell$. Intuitively, $h$ represents the number of consecutive events in a group, and $\ell$ is the number of groups per process, divided by 2. 
The graph depicted on the top left of Fig.~\ref{fig:ws_p2p_tw_3} is not an MSC, because it is undirected and there are multiple actions associated to the same event. Nonetheless, the connection with MSCs is quite intuitive, and formalized in Lemma~\ref{lem:p2pwsMSC}.

  We, now, specify how to build a digraph $G_{h,\ell} = (V(G_{h,\ell}), \arcs(G_{h,\ell}))$, from which our MSC $G_{h,\ell}^*$ will be obtained. 
The set of vertices  $V(G_{h,\ell})={\cal A} \cup {\cal B} \cup {\cal C}$
contains all the events of each process:  ${\cal A}=\{s^{i,j}_a,r^{i,j}_a \mid {1 \leq i \leq h, 1 \leq j \leq \ell} \}$, ${\cal B}=\{s^{i,j}_b,r^{i,j}_b \mid {1 \leq i \leq h, 1 \leq j \leq \ell} \}$, and ${\cal C}=\{s^{i,j}_c,r^{i,j}_c \mid {1 \leq i \leq h, 1 \leq j \leq \ell} \}$. %Each set will correspond to one of the three processes $a,b$ and $c$.
%For $x \in \{a,b,c\}$ ($a,b$ and $c$ will be the three processes), let ${\cal X}=\{s^{i,j}_x,r^{i,j}_x\}_{1 \leq i \leq h, 1 \leq j \leq \ell}$.
%Let $V(G_{h,\ell})={\cal A} \cup {\cal B} \cup {\cal C}$. 

For $x \in \{a,b,c\}$ and $y \in \{r,s\}$, we add the following arcs to $\arcs(G_{h,\ell})$, which will represent the ``timelines'' connecting events of each process:
\begin{enumerate}
    \item %For any $x \in \{a,b,c\}$, $y \in \{r,s\}$ and 
    for each group of $h$ events/messages and $1 \leq j \leq \ell$,  $Col_{x,y,j}= \{ (y^{i,j}_x,y^{i+1,j}_x) \mid 1 \leq i <h\}$; 
    
     %Let $Col = \bigcup_{x \in \{a,b,c\},y \in \{r,s\},1 \leq j \leq \ell} Col_{x,y,j}$.
    \item then, to link groups together %for any $x \in \{a,b,c\}$ and $y \in \{r,s\}$, let us add the arcs in 
    %of $Col_{x,y}=
    $\{ (y^{h,j}_x,y^{1,j+1}_x) \mid 1\leq j < \ell\}$;
    \item and finally, to link the phase of sendings with the one of receptions:  $(s^{h,\ell}_x,r^{1,1}_x)$. %of each $x \in \{a,b,c\}$.
\end{enumerate}
It remains to add the arcs that   correspond to the messages exchanged by the processes. Intuitively, each vertex $s^{i,j}_x$  corresponds to two messages sent by process $x$ to the other two processes (except for $j=1$ and $x=a$, in which case it will correspond to a single message), and each vertex $r^{i,j}_x$ will correspond to two messages received by process $x$ from the other two processes (except for $j=\ell$ and $x=c$, in which case it will correspond to a single message).
%$r^{i,j}_x$ (resp., $s^{i,j}_x$) will correspond to the $((1+2(j-2))h+2(i-1)+1)^{th}$ and $((1+2(j-2))h+2(i-1)+2)^{th}$ messages sent (resp., received) by process $x$. 
Formally: %, let us add the following arcs:
%\begin{enumerate}
%\setcounter{enumi}{3}
    %\item Let 
    \begin{align}
    E_{\mathcal{M}}= & \{ (s^{i,j}_a,r^{i,j}_b), (s^{i,j}_c,r^{i,j}_b), (s^{i,j}_c,r^{i,j}_a), (s^{i,j}_b,r^{i,j}_a),(s^{i,j}_b,r^{i,j}_c) \mid 1 \leq i \leq h, 1 \leq j \leq \ell\} \nonumber\\
    & \cup \{(s^{i,j+1}_a,r^{i,j}_c) \mid 1 \leq i \leq h, 1 \leq j < \ell\}.\label{eq:edges}
    \end{align}

    % \in [1,h], j \in [1,j]

%    
    %$= \{ (s^{i,j}_a,r^{i,j}_b), (s^{i,j}_c,r^{i,j}_b), (s^{i,j}_c,r^{i,j}_a), (s^{i,j}_b,r^{i,j}_a),(s^{i,j}_b,r^{i,j}_c), (s^{i,j+1}_a,r^{i,j}_c)\mid 1 \leq i \leq h, 1 \leq j < \ell\} \cup  \{ (s^{i,\ell}_a,r^{i,\ell}_b), (s^{i,\ell}_c,r^{i,\ell}_b), (s^{i,\ell}_c,r^{i,\ell}_a), (s^{i,\ell}_b,r^{i,\ell}_a),(s^{i,\ell}_b,r^{i,\ell}_c)\mid 1 \leq i \leq h\}$. 
%
 
%\end{enumerate}

\begin{figure}[t]
    \centering
    \includegraphics[width=\textwidth]{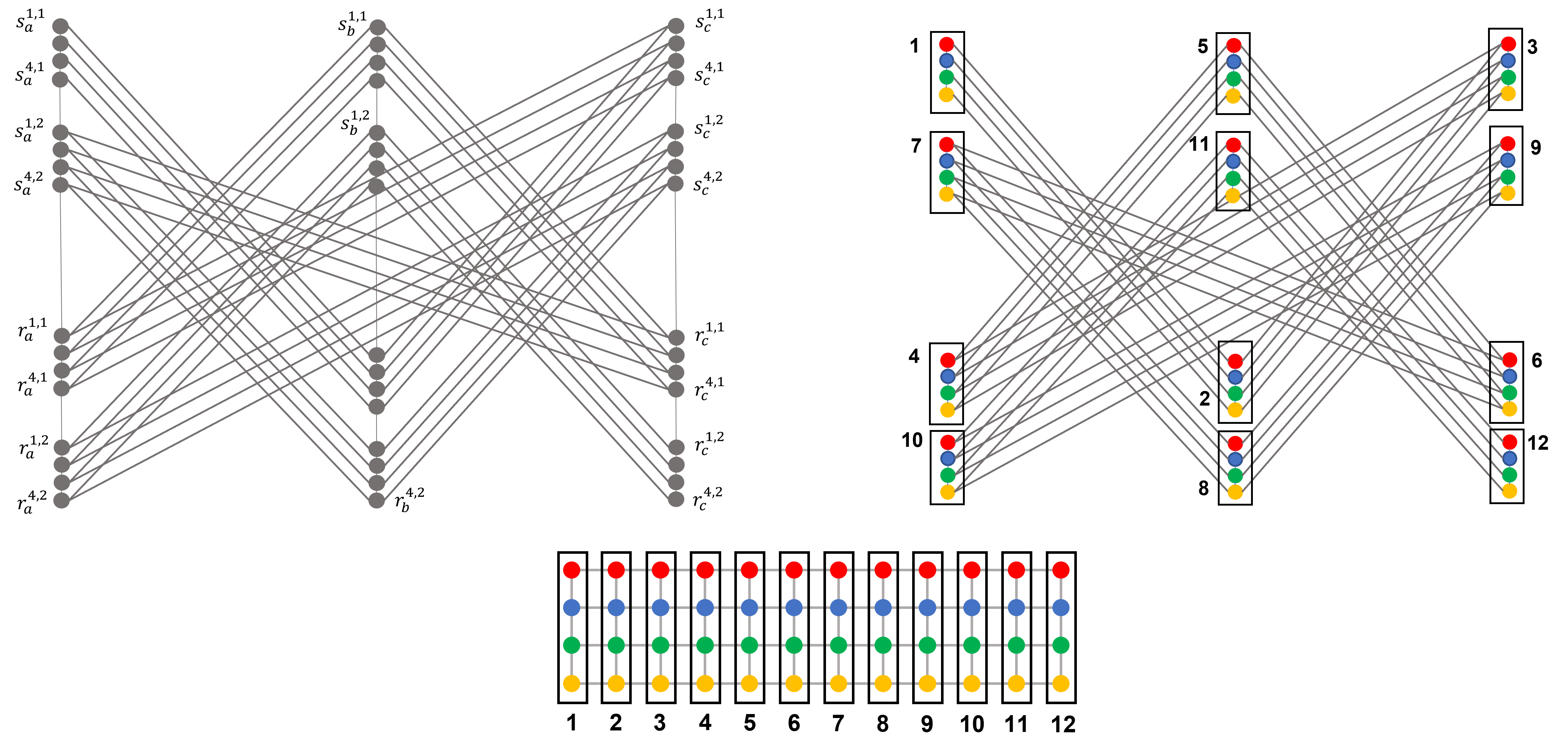}
    \caption{The undirected graph of $G_{4,2}$ (top left) with a $4 \times 12$ grid as a minor (top right and bottom). All arcs go from top to bottom. %$G$ is a minor of a weakly synchronous \pp MSC.
    } %In step 1, we delete some edges and label some groups of nodes, for convenience. In step 2, we simply rearrange those groups to show that they form a $4 \times 12$ grid.}
    \label{fig:ws_p2p_tw_3}
\end{figure}

\begin{wrapfigure}[13]{R}{0.2\textwidth}
    \centering
    \includegraphics[width=0.2\textwidth]{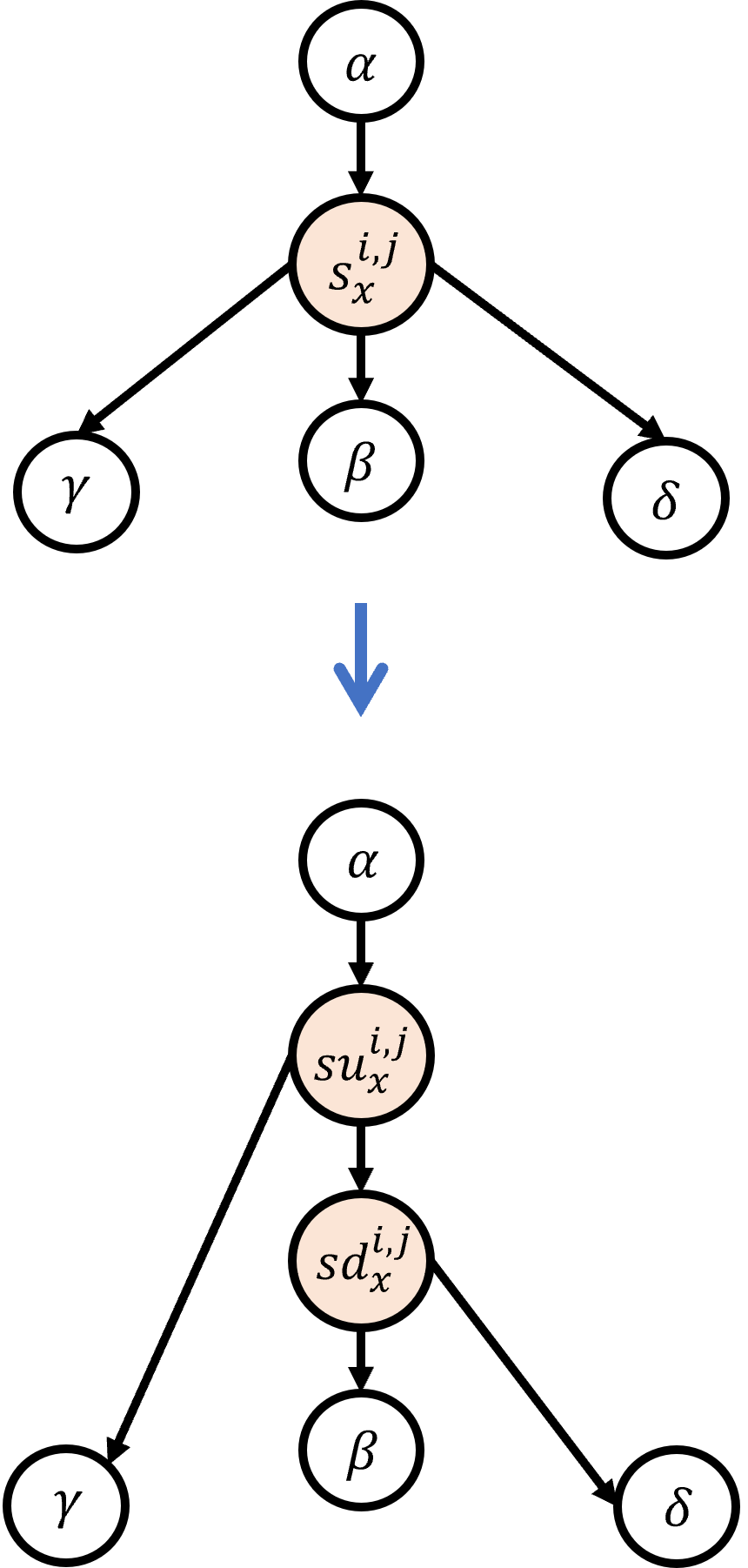}
    \caption{Transformation of Lemma~\ref{lem:p2pwsMSC}.}
    \label{fig:from_G_to_MSC}
\end{wrapfigure}

\begin{lemma}\label{lem:p2pwsMSC}
    For any $h, \ell \in \mathbb{N}^+$, $G_{h,\ell}$ is the minor of a graph arising from a weakly synchronous \pp MSC $G_{h,\ell}^*$ with $3$ processes and a single phase. 
\end{lemma}
\begin{proof}
 Fig.~\ref{fig:from_G_to_MSC}  exemplifies the transformation below. 
Note that some vertices of $G_{h,\ell}$ have degree $4$ while any MSC is a subcubic graph (i.e., every vertex has degree at most $3$). For every $s^{i,j}_x$ with degree $4$, let $\alpha$ (resp., $\beta$) be the in-neighbor (resp., out-neighbor) of $s^{i,j}_x$ in ${\cal P}_x$ and let $\gamma$ and $\delta$ be the other two neighbors of $s^{i,j}_x$. 
    %, \delta\}= N(r^{i,j}_x) \setminus \{\alpha,\beta\}$.
    Replace $s^{i,j}_x$ by two vertices $su^{i,j}_x$ and $sd^{i,j}_x$, with the $5$ arcs $(\alpha,su^{i,j}_x), (su^{i,j}_x,sd^{i,j}_x), (sd^{i,j}_x,\beta),(su^{i,j}_x,\gamma)$ and $(sd^{i,j}_x,\delta)$. Do a similar transformation for every $r^{i,j}_x$ with degree $4$. A similar transformation is done for the four vertices (with degree $3$) $s_b^{1,1}, s_c^{1,1},r_a^{h,\ell}$ and $r_b^{h,\ell}$.  
Let $G^*_{h,\ell}$ be the obtained digraph. 
  It is clear that $G^*_{h,\ell}$ is an MSC and that $G_{h,\ell}$ is a minor of $G^*_{h,\ell}$. 
    
    Note that for any $x \in \{a,b,c\}$, ${\cal X} \in \{{\cal A},{\cal B},{\cal C}\}$ induces a directed path ${\cal P}_x$ with first the vertices $s^{i,j}_x$ (in increasing lexicographical order of $(j,i)$) and then the vertices $r^{i,j}_x$ (in increasing lexicographical order of $(j,i)$).
The fact that $G^*_{h,\ell}$ is weakly synchronous with one phase directly follows the fact that, for every $x \in \{a,b,c\}$, the vertices $s,su$ and $sd$ (corresponding to sendings) appear before the vertices $r,ru$ and $rd$ (corresponding to receptions) in the directed path ${\cal P}_x$.

    Moreover, for every $x,y \in \{a,b,c\}$, $x \neq y$, the arcs from ${\cal P}_x$ to ${\cal P}_y$ are all parallel (i.e.,  for every arc $(u,v)$ and $(u',v')$ from ${\cal P}_x$ to ${\cal P}_y$, if $u$ is a predecessor of $u'$ in ${\cal P}_x$, then $v$ is a predecessor of $v'$ in ${\cal P}_y$). This implies that $G^*_{h,\ell}$ is \pp. \qed
\end{proof}

Note that, for  fixed $i\leq h$ and $j<\ell$, $P_{i,j}=(s^{i,j}_a,r^{i,j}_b,s^{i,j}_c,r^{i,j}_a, s^{i,j}_b,r^{i,j}_c,s^{i+1,j}_a)$ is a  (undirected) path with $6$ arcs linking $s^{i,j}_a$ to $s^{i,j+1}_a$. From this, it is not difficult to see that $G_{h,\ell}$ admits a grid of size $h \times 6\ell$ as a minor, which is the content of next lemma (see Fig.~\ref{fig:ws_p2p_tw_3} for an example). 

Let $tw(G_{h,\ell})$ be the treewidth of the underlying undirected graph of $G_{h,\ell}$.

\begin{lemma}\label{lem:unboundTw}
    For any $h, \ell \in \mathbb{N}^*$, $tw(G_{h,\ell}) \geq \min \{h,6\ell\}$. 
\end{lemma}
\begin{proof}
    The subgraph obtained from $G_{h,\ell}$ by  keeping the arcs in item  $1$ and Equation \ref{eq:edges}: $G'_{h,\ell}=(V(G_{h,\ell}),E_{\mathcal{M}} \cup \bigcup_{x \in \{a,b,c\}, y \in \{r,s\}, 1 \leq j \leq \ell} Col_{x,y,j})$, is a $h \times 6\ell$ grid. %{\color{red} (Nico : must be defined?)}. 
    From \cite{bodlaender:partialK}, we know that $tw(G'_{h,\ell}) \geq \min \{h,6\ell\}$ and, since treewidth is closed under subgraphs~\cite{bodlaender:partialK}, $tw(G_{h,\ell}) \geq tw(G'_{h,\ell}) \geq \min \{h,6\ell\}$. \qed
\end{proof}

We can then easily derive the main result for this section.

%Note that, for $h \geq 6$, we can prove that $G^*_{h,\ell}$ is not \mb, which is not surprising since it is known that the class of weakly synchronous \mb MSCs has bounded treewidth. On the other hand:

\begin{theorem}
\label{th:tw}
    The class of weakly synchronous \pp MSCs with three processes (and a single phase) has unbounded treewidth.
\end{theorem}
\begin{proof}
From Lemma~\ref{lem:p2pwsMSC}, $G^*_{h,\ell}$ is a weakly synchronous \pp MSC with $3$ processes and $G_{h,\ell}$ is a minor of $G^*_{h,\ell}$. Hence, from Lemma~\ref{lem:unboundTw} and the fact that the treewidth is minor-closed~\cite{bodlaender:partialK}, we get that $tw(G^*_{h,\ell})\geq \min \{h,6\ell\}$. \qed
\end{proof}

Notice that, a similar technique, this time exploiting four processes instead of three, can be used to show that we can build a weakly synchronous \pp MSC that can be contracted to whatever graph.

\begin{theorem}
\label{th:minor}
    Let $H$ be any graph. There exists a weakly synchronous \pp MSCs with four processes that admits $H$ as minor.
\end{theorem}
%\davideinline{I think this proof can go in the appendix.}
\begin{proof}
Let $V(H)=\{v_1,\cdots,v_h\}$ and $E(H)=\{e_1,\cdots,e_{\ell}\}$. Take graph $G_{h,\ell}$ defined above. Add a new directed path $(d_1,\cdots,d_{\ell})$  (which corresponds to the fourth process). Finally, for every $1 \leq j \leq \ell$, and edge $e_j=\{v_i,v_{i'}\} \in E(H)$,  add two arcs $(r^{i,j}_a,d_j)$ and $(r^{i',j}_a,d_j)$. Let $G$ be the obtained graph.

Using similar arguments as in the proof of Lemma~\ref{lem:p2pwsMSC}, $G$ arises from a weakly synchronous \pp MSC with $4$ processes. 
Now, to see that $H$ is a minor of $G$, first remove all ``vertical'' arcs from $G$. Then, for every $1 \leq i \leq h$, contract the path $\bigcup_{1 \leq j \leq \ell}P_{i,j}$ into a single vertex (corresponding to $v_i$), and finally contract the arc $(r^{i',j}_a,d_j)$ for every edge $e_j=\{v_i,v_{i'}\}$. These operations lead to $H$.
\qed
\end{proof}

\section{Reachability for weakly synchronous \pp systems with 3 machines}\label{sec:reach}

In \cite{BolligGFLLS21}, it is shown that the control state reachability problem
for weakly \pp synchronous systems with at least 4 processes is undecidable. The result is obtained via a reduction of the Post correspondence problem. In the same paper, following from the boundedness of treewidth,  it is also shown that reachability  is decidable for systems with 2 processes. 
The arguments easily adapt to show the same results for the emptiness problem instead.
The decidability of reachability, or emptiness, remained open for systems with 3 processes.
We already showed that the treewidth of weakly synchronous \pp MSCs is unbounded  for 3 processes. 
But, this result alone is not enough to prove undecidability, still it gives us a hint on how to conduct the proof. Indeed,  inspired by the proof of the unboundedness of the treewidth, we provide a reduction from the emptiness problem for a FIFO automaton $\Sys_1$ (undecidable, see Lemma~\ref{lem:FIFO-automaton-emptiness-is-not-decidable})
to the emptiness problem for a weakly synchronous system $\Sys_3$ with three machines. The reduction makes sure that there is an accepting run of $\Sys_1$  if and only if there is one for $\Sys_3$, which shows the undecidability of the emptiness problem
for weakly synchronous systems with three machines.%, thanks to Lemma~\ref{lem:FIFO-automaton-emptiness-is-not-decidable}.

Let $\Sys_1= (A)$, with $A= (Loc, \delta, \ell^0, \ell^{acc})$ be a communicating system with a single process over $\Msg$. 
We will consider only automata that, from any state, have at most one non-epsilon outgoing transition, and no self loops (i.e., transitions that start and land in the same state). More precisely, we prove that any system can be encoded into one that satisfies this additional property while accepting the same language (see the corresponding encoding in Appendix \ref{apx:epsilon}).

\begin{figure}[t]
    \begin{center}
        \begin{tikzpicture}[>=stealth,node distance=2cm,shorten >=1pt,
        every state/.style={text=black, scale =0.8}, semithick,
        font={\fontsize{8pt}{12}\selectfont},
        scale = 1
        ]
        \begin{scope}[->]
            \node[state,initial,initial text={}] (q0)  {\large $\ell_b^0$};
            \node[state, right of=q0] (q1)  {\large $\ell_b^?$};
            \node[state, accepting, right of=q1] (q2) {\large $\ell_b^{acc}$};
            \node[state, above of=q0] (q3) {\large $\ell_{b_0}^m$};
            \node[state,  below of=q1] (q4) {\large $\ell_{b_?}^m$};

        \path (q0) edge node [above] {$\varepsilon$} (q1);
        \path (q0) edge [bend left] node [left] {$\send{b}{a}{\msg}$} (q3);
        \path (q3) edge [bend left] node [right] {$\send{b}{c}{\msg}$} (q0);
        \path (q1) edge node [above] {$\varepsilon$}(q2);
            \path (q1) edge [bend left] node [right] {$\rec{a}{b}{\msg}$} (q4);
        \path (q4) edge [bend left] node [left] {$\rec{c}{b}{\msg}$} (q1);
        \node[thick] at (-1.1,0) {$A_b$};
        \end{scope}
    
        \begin{scope}[->, shift={(5.5,0)}]
            \node[state,initial,initial text={}] (q0)  {\large $\ell_c^0$};
            \node[state, right of=q0] (q1)  {\large $\ell_c^?$};
            \node[state,accepting, right of=q1] (q2) {\large $\ell_c^{acc}$};
            \node[state, above of=q0] (q3) {\large $\ell_{c_0}^m$};
            \node[state, below of=q1] (q4) {\large $\ell_{c_?}^m$};

        \path (q0) edge node [above] {$\varepsilon$} (q1);
        \path (q0) edge [bend left] node [left] {$\send{c}{b}{\msg}$} (q3);
        \path (q3) edge [bend left] node [right] {$\send{c}{a}{\msg}$} (q0);
        \path (q1) edge node [above] {$\varepsilon$}(q2);
            \path (q1) edge [bend left] node [right] {$\rec{b}{c}{\msg}$} (q4);
        \path (q4) edge [bend left] node [left] {$\rec{a}{c}{\msg}$} (q1);
        \node[thick] at (-1.1,0) {$A_c$};
        \end{scope}
    
    \end{tikzpicture}
    \captionof{figure}{Sketch of $A_b$ and $A_c$ of $\Sys_3$ (only a single message $m$ is considered).}
    \label{fig:encoding}
    \end{center}
\end{figure}
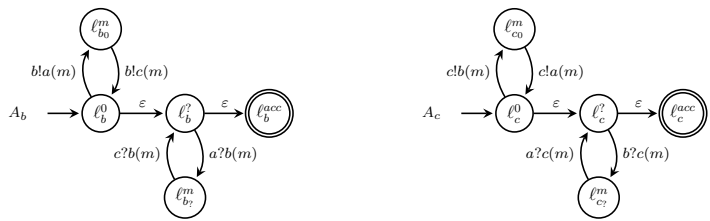

We provide an encoding of the FIFO automaton $\Sys_1$ into  the system $ \Sys_3 = (A_a, A_b, A_c)$ over $\Msg \cup \{D\}$, where $D$ is an additional special message called the \emph{dummy message}. We  show that $\Sys_3$ is weakly synchronous, and that $\Lang{\Sys_1}\neq\emptyset$ if and only if $\Lang{\Sys_3}\neq\emptyset$.
Processes $b$ and $c$ (see Fig. \ref{fig:encoding}) are used as forwarders so that messages circulate as in Fig. \ref{fig:ws_p2p_tw_3}. Basically, process $b$ (resp., process $c$) goes through two phases, the first one in which messages are sent to $a$ and $c$ (resp., $a$ and $b$), and the second in which messages can be received. In Fig. \ref{fig:encoding}, there should be one state $\ell_{b_0}^m$ (resp., $\ell_{b_?}^m$), which is the in and out-neighbor of $\ell_b^0$ (resp.,  $\ell_b^?$), per message $m \in \Msg \cup \{D\}$.
Formally,   $A_b= (Loc_b, \delta_b, \ell_b^0, \ell_b^{acc})$ where 
\begin{align*}
Loc_b = & \{\ell_b^0,  \ell_b^?,  \ell_b^{acc}\} \cup \{\ell_{b_0}^m, \ell_{b_?}^m \mid m \in \Msg \cup \{D\} \} \\
\delta_b= & \{ (\ell_b^0, \varepsilon,  \ell_b^?), ( \ell_b^?, \varepsilon, \ell_b^{acc})\} ~\cup~ 
  \{ (\ell_b^0, \send{b}{a}{m}, \ell_{b_0}^m), (\ell_{b_0}^m, \send{b}{c}{m}, \ell_b^0), \\
 & \quad\quad ( \ell_b^?, \rec{b}{a}{m}, \ell_{b_?}^m),(\ell_{b_?}^m, \rec{b}{c}{m},  \ell_b^?) \mid m\in \Msg\cup \{D\} \}
\end{align*} 
and symmetrically  $A_c= (Loc_c, \delta_c, \ell_c^0, \ell_c^{acc})$ where 
\begin{align*}
Loc_c = & \{\ell_c^0,  \ell_c^?,  \ell_c^{acc}\} \cup \{\ell_{c_0}^m, \ell_{c_?}^m \mid m \in \Msg\cup \{D\} \} \\
\delta_c = & \{ (\ell_c^0, \varepsilon, \ell_c^?), (\ell_c^?, \varepsilon, \ell_c^{acc})\} \cup 
 \{ (\ell_c^0, \send{c}{b}{m}, \ell_{c_0}^m), (\ell_{c_0}^m, \send{c}{a}{m}, \ell_c^0), \\
 & \quad\quad (\ell_c^?, \rec{c}{b}{m}, \ell_{c_?}^m),  (\ell_{c_?}^m, \rec{c}{a}{m}, \ell_c^?) \mid m\in \Msg\cup \{D\} \}.
\end{align*}

Process $a$  mimics the behavior of $A$. Fig. \ref{fig:automata_conversion} shows an example of how $A_a$ is built, starting from $A$. At a high level, $A_a$ is composed of two parts: the first simulates $A$, and the second (after state $\ell^{D}_a$) receives all messages sent by $b$ and $c$.
In the first part of $A_a$, each send action of $A$ is replaced by a send action addressed to process $b$, and each reception of $A$ is replaced by a send action to process $c$. We then use some dummy messages to ensure that our encoding works properly. Roughly, we force $A_a$ to send a dummy message to $b$ after each message sent to $c$, and we let $A_a$ send any number of dummy messages to $c$ right before each message sent to $b$, or right before entering the "receiving phase" of $A_a$, where messages from $b$ and $c$ are received.
Similarly, after $A_a$ sends a dummy message to $b$, it is also allowed to send two other dummy messages (the first one to $c$ and the second one to $b$) an unbounded number of times. 
%These rules might seem bizarre at first, but they will allow us to prove that there is an accepting run for $\Sys_1$  if and only if there is one for $\Sys_3$. 
Formally, $A_a= (Loc_a, \delta_a, \ell^0, \ell_a^{acc})$, where: 
\begin{align*}
    Loc_a =  &Loc \cup \{\ell_{t_1}, \ell_{t_2} \mid t=(\ell, ?m, \ell') \in \delta\} \cup \\
    &\{ \ell^{D}_a, \ell_a^?,  \ell_a^{acc}\} \cup \{ \ell_{a_?}^m \mid m \in \Msg\cup \{D\} \} \\
    \delta_a = &\{  (\ell, \send{a}{b}{m}, \ell'), (\ell, \send{a}{b}{D}, \ell)
    \mid  (\ell, !m, \ell') \in \delta \} 
    \cup \\
    & \{  (\ell, \send{a}{c}{m}, \ell_{t_1}),
          (\ell_{t_1}, \send{a}{b}{D}, \ell_{t_2}), \\
    &\quad\quad (\ell_{t_2}, \send{a}{c}{D}, \ell_{t_1}),
                (\ell_{t_2}, \varepsilon, \ell') 
                \mid  t=(\ell, ?m, \ell') \in \delta \} 
    \cup \\
    & \{  (\ell, \varepsilon, \ell') \mid  (\ell, \varepsilon, \ell') \in \delta \}
     \cup \\
    & \{  (\ell^{acc}, \varepsilon, \ell^{D}_a), (\ell^{D}_a, \send{a}{c}{D}, \ell^{D}_a) \} \cup \\
   &  \{ (\ell^{D}_a, \varepsilon, \ell_a^?), (\ell_a^?, \varepsilon, \ell_a^{acc}) \} \cup \\
    & \{  (\ell_a^?, \rec{a}{c}{m}, \ell_{a_?}^m),  (\ell_{a_?}^m, \rec{a}{b}{m}, \ell_a^?) \mid m\in \Msg\cup \{D\} \} \\
   % \ell_a^{0} = &\ell^{0}
\end{align*}
In Fig. \ref{fig:automata_conversion}, colors  show the mapping of states from an instance of  $A$ to the corresponding automaton $A_a$. Fig.~\ref{fig:s3_acc_run} illustrates an accepting run of some system $\Sys_1$ and one of the corresponding accepting runs of the associated $\Sys_3$.

\begin{figure}[t]
    \centering
    \includegraphics[width=\textwidth]{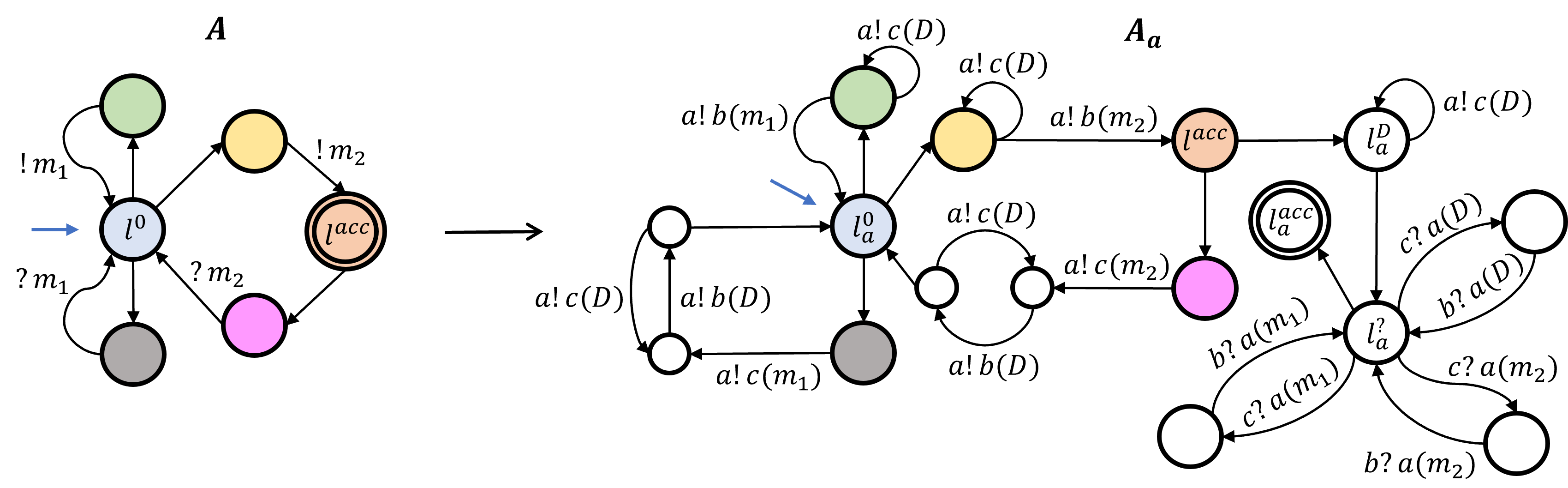}
    \caption{The automaton $A_a$ for the system $\Sys_3$, built from the automaton $A$ of $\Sys_1$. Arcs without actions represent $\epsilon$ transitions.}
    \label{fig:automata_conversion}
\end{figure}
 
Given a sequence of actions $!m$ and $?m$, where $m$ can be any message, we call it a FIFO sequence if
\begin{enumerate*}[label={(\roman*)}]
    \item all messages are received in the order in which they are sent, and
    \item no message is received before being sent.
\end{enumerate*}
We relax this definition to  talk about sequences of send actions $\send{a}{b}{m}$ and $\send{a}{c}{m}$ taken by $a$ (in the first part of the automaton $A_a$); in particular, we say that such a sequence $\gamma'$ is FIFO if, when interpreting each $\sact{a}{b}{m}$ and $\sact{a}{c}{m}$ action as $!m$ and $?m$, respectively, $\gamma'$ is a FIFO sequence. Dummy messages are  used to enforce that the sequence of send actions taken by $A_a$ in an accepting run of $\Sys_3$ is FIFO.

%Let $\Sys_1 = (A)$ be a FIFO automata and let $\Sys_3 = (A_a, A_b, A_c)$ be the weakly synchronous system obtained from $\Sys_1$ using the encoding above. 

\begin{figure}[t]
    \centering
    \includegraphics[width=0.4\textwidth]{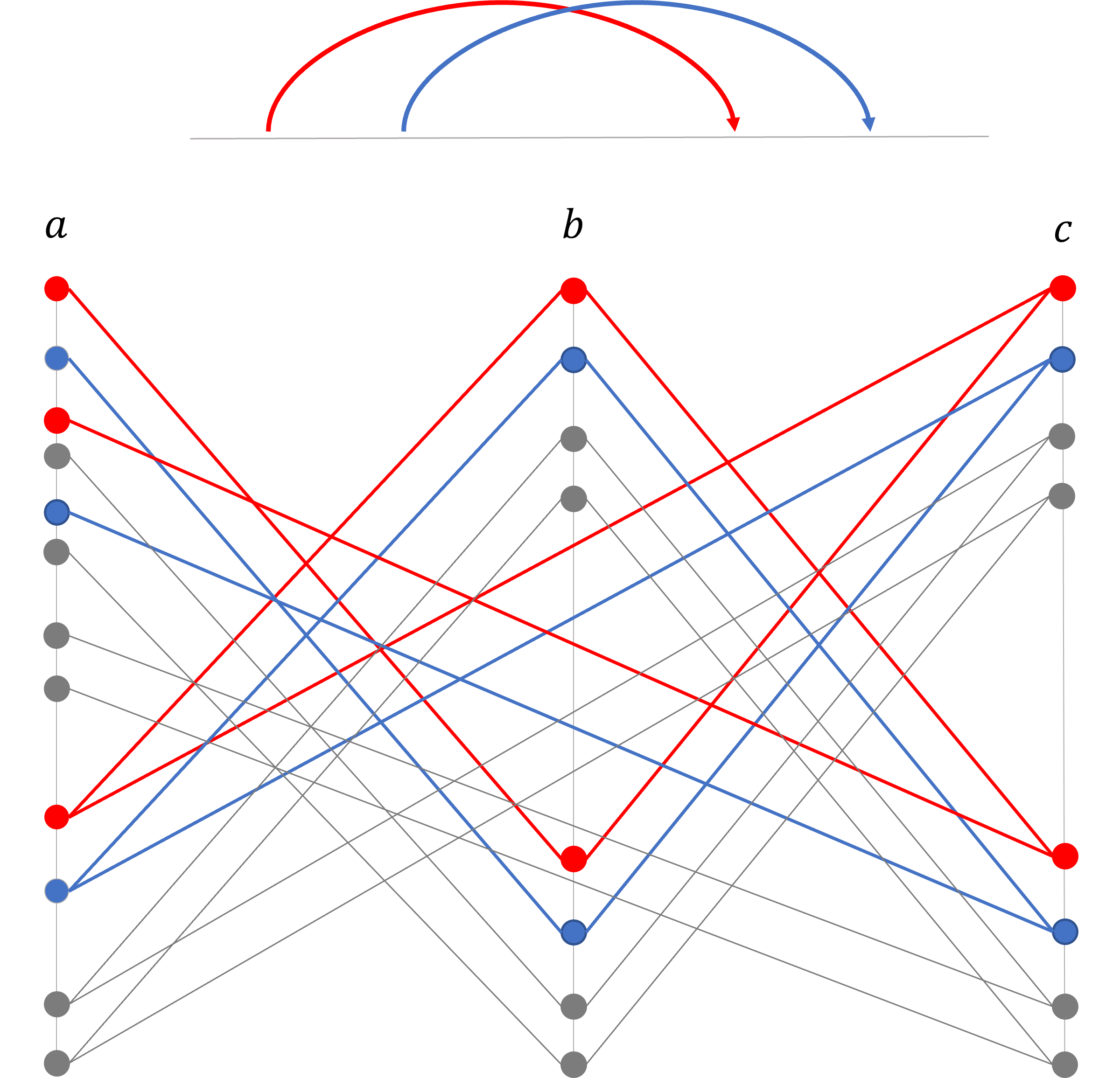}
    \caption{Above, a run with two messages for some system $\Sys_1$ with a single process (timeline drawn horizontally). Below, one possible corresponding MSC realized by the associated $\Sys_3$. Gray lines correspond to dummy messages.}
    \label{fig:s3_acc_run}
\end{figure}

\begin{restatable}{theorem}{undReach}
\label{thm:und_reach}
    There is an accepting run of $\Sys_1$ if and only if there is an accepting run of $\Sys_3$.
\end{restatable}
\begin{sketch}
    We only provide a sketch of the proof, which is quite convoluted and requires several intermediate lemmata. The full proof is  in Appendix~\ref{apx:proofs}.

    \noindent($\Rightarrow$) We design Algorithm \ref{alg:complete_dummies}, which  takes an accepting run $\sigma$ of $S_1$, and returns an accepting run $\mu$ for $S_3$. At a high level, Algorithm \ref{alg:complete_dummies} takes the sequence of actions taken by $A$ in $\sigma$, rewrites each $!m$ and $?m$ action as $\sact{a}{b}{m}$ and $\sact{a}{c}{m}$, and then adds some actions related to dummy messages. We first show that the sequence of actions $\gamma'$ returned by Algorithm \ref{alg:complete_dummies} is a sequence of send actions that takes $A_a$ of $\Sys_3$ from state $\ell^0$ to $\ell^{acc}$ (note that this is not the final state of $A_a$, see Fig.~\ref{fig:automata_conversion} for an example). We then show that $\gamma'$ is a FIFO sequence, and prove that there exists an accepting run of $\Sys_3$ in which $A_a$ starts by executing exactly the sequence of actions in $\gamma'$.
    Finally, we show that Algorithm \ref{alg:complete_dummies} always terminates.

    \noindent ($\Leftarrow$) Let $\mu$ be an accepting run of $\Sys_3$, from which we show that it is easy to build a sequence of actions $\gamma$ taken by $A$ in an accepting run of $\Sys_1$. Let $\gamma'$ be the sequence of send actions taken by $A_a$ in the accepting run $\mu$. The first step is to show that $\gamma'$ is a FIFO sequence. The three automata $A_a$, $A_b$, and $A_c$ are built so to ensure that $\gamma'$ is always a FIFO sequence. This is closely related to the shape of the MSCs associated to accepting runs of $\Sys_3$; these MSCs exploit the same kind of pattern seen in Section~\ref{sec:tw_ws_pp} to bounce messages back and forth between the three processes. We then prove that, if we ignore actions related to dummy messages in $\gamma'$ and interpret each $\sact{a}{b}{m}$ and $\sact{a}{c}{m}$ action as $!m$ and $?m$, we get a sequence of actions $\gamma$ that takes $A$ from its initial state $\ell^0$ to its final state $\ell^{acc}$ in an accepting run of $\Sys_1$.
\end{sketch}

    \label{sec:undecidability}
    
\begin{algorithm}
\caption{Let $\sigma$ be an accepting run of $\Sys_1$, and $\alpha^\sigma$ be the sequence of $n$ actions taken by $A$ in $\sigma$. We use $\alpha^\sigma(i)$ to denote the $i$-th action of $\alpha^\sigma$.}
    \label{alg:complete_dummies}
    \begin{multicols}{2}
    \begin{algorithmic}[1]
    \STATE $\gamma'$ $\gets$ empty list
    \STATE Queue $\gets$ empty queue
    \FOR{$i$ from $1$ to $n$}
    \STATE $action$ $\gets$ $\alpha^{\sigma}(i)$
    \IF{$action = !x$}
        \WHILE{first(Queue) = $D$}
            \STATE add $\sact{a}{c}{D}$ to $\gamma'$
            \STATE dequeue $D$ from Queue
        \ENDWHILE
        \STATE add $\sact{a}{b}{x}$ to $\gamma'$
        \STATE enqueue $x$ in Queue
    \ELSIF{$action = ?x$}
        \STATE add $\sact{a}{c}{x}$ to $\gamma'$
        \STATE dequeue $x$ from Queue
        \STATE add $\sact{a}{b}{D}$ to $\gamma'$
        \STATE enqueue $D$ in Queue
        \IF{Queue does not contain only $D$}
            \WHILE{first(Queue) = $D$}
                \STATE add $\sact{a}{c}{D}$ to $\gamma'$
                \STATE dequeue $D$ from Queue
                \STATE add $\sact{a}{b}{D}$ to $\gamma'$
                \STATE enqueue $D$ in Queue
            \ENDWHILE
        \ENDIF
    \ENDIF
    \ENDFOR
    \WHILE{first(Queue) = $D$}
        \STATE add $\sact{a}{c}{D}$ to $\gamma'$
        \STATE dequeue $D$ from Queue
    \ENDWHILE
    \RETURN $\gamma'$;
    \end{algorithmic}
    \end{multicols}
\end{algorithm}

The following result immediately follows from Lemma~\ref{lem:FIFO-automaton-emptiness-is-not-decidable}
and Theorem~\ref{thm:und_reach}.

\begin{theorem}\label{thm:main-theorem}
    The emptiness problem for weakly synchronous communicating systems with three processes is undecidable.
\end{theorem}

% \begin{proof}
%     From Theorem \ref{thm:und_reach}, it follows that the emptiness problem for weakly synchronous communicating systems with three processes is undecidable. By contradiction, suppose it was decidable: given any system $\Sys_1$ with a single process and a queue, we could build the corresponding weakly synchronous system $\Sys_3$ with three processes, solve the emptiness problem for $\Sys_3$, and return the same answer for $\Sys_1$. This is a contradiction, since it would imply the decidability of the emptiness/reachability problem for systems with a single process and a queue, which we know to be undecidable. The last step is to show that the undecidability of the emptiness problem for weakly synchronous communicating systems with three processes implies the undecidability of reachability for those systems. By contradiction, suppose reachability was decidable, and let $\Sys_3^{ws}$ be any weakly synchronous system with three processes and final global state $\ell$. If $\ell$ is not reachable, then the emptiness problem has also a negative answer. If $\ell$ is reachable, there are two possibilities:
%     \begin{itemize}
%         \item $\Sys_3$ can get to the final state $\ell$ without any unmatched message, which directly implies a positive answer for the emptiness problem, or
%         \item $\Sys_3$ gets to the final state $\ell$ with some unmatched messages, which is not enough to claim that there exists an accepting run for $\Sys_3$. 
%     \end{itemize}
%     \qed
% \end{proof}
Notice that our results extend to  causally ordered (CO) communication, since an MSC is weakly synchronous if and only if it is weakly synchronous CO (see Appendix \ref{apx:ws_p2p_co}).

\begin{corollary}
The emptiness problem for causal order communicating systems with three processes is undecidable.
\end{corollary}

\section{Conclusion}\label{sec:concl}
% !TEX root = ../conference.tex

We showed the undecidability of the reachability of a configuration for weakly synchronous systems with three processes or more. The main contribution lies in the technique used to achieve this result.
We first show that the treewidth of the class of weakly synchronous MSCs is unbounded, by proving that it is always possible to build such an MSC with an arbitrarily large grid as minor.
Then, a similar construction is employed to provide an encoding of a FIFO automaton into a weakly synchronous system with three processes, allowing to show that reachability of a configuration is undecidable.

%\todo[inline,caption={}, backgroundcolor=lime!50]{    
%    D: ISSUES: 
%    \begin{enumerate}\itemsep=0.5ex
%        \item Define accepting run with or without emptying the queues. If they are empties, proofs in the reachability part should be much clearer.
%        \item Problem of accepting run of a single machine with a single queue... our definition of MSC does not admit messages sent to itself. What could be the easiest approach? If we consider two machines that communicate with two queues, they are Turing equivalent and we avoid the MSC problem, but we should completely rewrite the reduction rules and all the reachability proof... not the way to go imo.
%        \item Does the current definition of MSCs admint self messages? Do the previous results work if we allow self messages?
%    \end{enumerate}
%}

\clearpage
% Bibliography
\bibliography{bibfile}
\newpage

\appendix

\section{Tree-decomposition and treewidth}\label{app:defs}
% !TEX root = ../conference.tex

\begin{definition} 
    A \emph{ tree-decomposition} of a graph $G=(V,E)$ is a pair $(T, {\cal X}=\{ X_{t}\mid t \in V(T)\})$ such that $T$ is a tree, and ${\cal X}$ is a set of subsets of $V$, one for each node of $T$, such that:
    \begin{enumerate*}[label={(\roman*)}]
        \item $\bigcup_{t\in V(T)} X_{t} = V(G)$;
        \item for every $\{u,v\}\in E(G)$, there exists $t \in V(T)$ such that $u, v\in X_{t}$;
        \item for every $v \in V(G)$, the set $\{t \in V(T) \mid v \in X_t\}$ induces a subtree of $T$.
    \end{enumerate*}

    The \emph{width} of a tree decomposition $(T, {\cal X})$ is $\max_{t \in V(T)} |X_t|-1$, i.e., the size of the largest set $V$ minus one. The \emph{treewidth} $tw(G)$ of $G$ is the minimum width over all possible tree decompositions of $G$. 
    %If $T$ is a path, then we call $(T, {\cal X})$ a \emph{path decomposition} of $G$, and the \emph{pathwidth} $pw(G)$ of $G$ is the minimum width over all possible path decompositions of $G$.
\end{definition}

The following well-known result in graph theory gives a connection between the notions of treewidth and minor.

\begin{theorem}~\cite{bodlaender:partialK} %[Minor theorem]
\label{thm:minor_tw}
    If $G$ is a minor of $H$, then $tw(G) \le tw(H)$.%, and $pw(G) \le pw(H)$.
\end{theorem}

%\section{Proof of Proposition \ref{prop:ws_p2p_4_tw} proof - continuation}
%\label{apx:is-ws-p2p}
%\input{apx-is-ws-p2p.tex}

\section{Automata with epsilon transitions}
\label{apx:epsilon}
% !TEX root = ../conference.tex
Given a communicating automaton $A$, we build an equivalent one with epsilon-transitions such that 
\begin{enumerate*}[label={(\roman*)}]
    \item from each state there are either only epsilon-transitions or a single transition labeled with a letter from $\Sigma$,
    \item and there are no states with a transition that lands in the same state (i.e., a self-loop).
\end{enumerate*}

Let $A = (Loc, \delta, \ell^0, \ell^{acc})$ be a communicating automaton for process $p$.
Its encoding into an automaton with single non-epsilon transitions is the automaton $A^{\varepsilon} = (Loc^{\varepsilon}, \delta^{\varepsilon}, \ell^{0}, \ell^{acc})$ where 
\begin{itemize}
\item $Loc^{\varepsilon} = Loc \cup \{\ell^t \mid t \in \delta\}$
\item 
%\begin{align*}
$\delta^{\varepsilon} =  \{(\ell, \varepsilon, \ell_t), (\ell_t, a, \ell') \mid t = (\ell,a,\ell') \in \delta \} $
%& \cup \{(\ell^{acc}, \varepsilon, \ell^{acc, \varepsilon}) \} \\ 
%& \cup \{(\ell^{acc, \varepsilon}, \rec{p}{p}{m}, \ell^{acc, \varepsilon})  \mid m \in \Msg \}
%\end{align*}

 \end{itemize}
Let $\Sys^{\varepsilon}$ be the system obtained from $\Sys$ where each for each of the processes $p \in \Procs$ we take the corresponding encoding $A_p^{\varepsilon}$.

Immediately from the definition of accepting run we can see that $\Lang{\Sys^{\varepsilon}} = \Lang{\Sys}$.

\section{Proofs for Section \ref{sec:undecidability}}
\label{apx:proofs}
This section is devoted to the proof of Theorem~\ref{thm:und_reach}, which we restate below.

\undReach*
\begin{proof}
    ($\Rightarrow$) Follows from Lemma~\ref{lem:acc_single_to_three}, which uses Lemmata~\ref{lem:fifo_realiz} to \ref{lem:alg_termination}. 
    
    \noindent ($\Leftarrow$) Follows from Lemma~\ref{lem:acc_m_iff_c}, which uses  Lemmata~\ref{lem:send_order} and \ref{lem:no_rec_bef_send}. All Lemmata are stated and proved below.
    \qed
\end{proof}

Let $\mu$ be an accepting run of a system $\Sys$, over a set $\Msg$ of messages and a set $\Procs$ of processes, on an MSC $\msc = (\Events,\procrel,\lhd,\lambda)$. For $p \in \Procs$, we will use $\alpha_p^{\mu}$ to denote the sequence of actions, ignoring $\varepsilon$-actions, taken by $A_p$ in the run $\mu$; more formally, $\alpha_p^{\mu}$ is the sequence of actions in $\{\tlabel(t) \;|\; t \in \mu(e), e \in \Events_p,\tlabel(t) \neq \varepsilon\}$, ordered according to the $\procrel_P$ relation, i.e., given $a_1=\tlabel(\mu(e_1))$, $b_2=\tlabel(\mu(e_2))$, such that $e_1 \procrel_p e_2$, then $a_1$ is right before $a_2$ in $\alpha_p^{\mu}$, which we abbreviate as $a_1 \xdashrightarrow{p} a_2$ (or simply $a_1 \dashrightarrow a_2$, when the process $p$ is clear from the context). 
The lightweight notation $\alpha_p$ will be used when the run $\mu$ is clear from the context, and we omit $p$ when the system only has one process. The $j$-th action in $\alpha_p^{\mu}$ will be denoted by $\alpha_p^{\mu}(j)$. 
In this section, when talking about an accepting run of a system $\Sys = ((A_p)_{p\in\procSet})$ on an MSC $M$, we will often not even mention $M$, and only focus on the sequence of actions taken by the automata $A_p$.

Let $\Sys_1 = (A)$ be any communicating system with a single process and one queue, and  $\Sys_3 = (A_a, A_b, A_c)$ be the weakly synchronous system obtained from $\Sys_1$ with the reduction described in Section~\ref{sec:reach}.  For $p \in \{a,b,c\}$, we  use  $!\alpha_p^{\mu}$ and $?\alpha_p^{\mu}$ to denote the sequence of send actions and, respectively, receive actions taken by $A_p$. Note that $\alpha_p^{\mu} = !\alpha_p^{\mu} + ?\alpha_p^{\mu}$, where $+$ is the concatenation of two sequences, since $\Sys_3$ is a weakly synchronous system with one phase. 
\begin{restatable}{lemma}{fifoRealiz}
    \label{lem:fifo_realiz}
        Let $\gamma=a_1 \ldots a_k$ be a sequence of send actions taken by $A_a$ to get from $\ell_a^0$ to $\ell^?_a$, where $A_a$ is also allowed to take extra $\varepsilon$-actions in any state. If $\gamma$ is a FIFO sequence, there is an accepting run $\mu$ of $\Sys_3$ such that $!\alpha^\mu_a=\gamma$.
\end{restatable}
\begin{proof}
    Suppose $\gamma$ is a FIFO sequence. It follows that $a$ must send messages to $b$ in the same order $X=m_1 \ldots m_k$ as $a$ sends messages to $c$. We will now build an accepting run $\mu$ of $\Sys_3$ such that $!\alpha^\mu_a=\gamma$. For $A_b$ ($A_c$), we can construct $!\alpha^\mu_b$ ($!\alpha^\mu_c$) by sending messages to $a$ and $c$ ($b$ and $a$) in order $X$. For $p \in \{a,b,c\}$, $?\alpha^\mu_p$ is constructed by receiving messages from the other two processes in order $X$. $\alpha^\mu_a$, $\alpha^\mu_b$, and $\alpha^\mu_c$ all lead the corresponding automaton to the final state, and all messages that are sent are also received, so $\mu$ is an accepting run of $\Sys_3$.
    \qed
\end{proof}

We now present Algorithm \ref{alg:complete_dummies}, which essentially takes an accepting run of $S_1$, and returns an accepting run for $S_3$. The correctness of the algorithm is proved in a few steps.
    
\begin{restatable}{lemma}{algRealizSeq}
\label{lem:alg_realiz_seq}
    Let $\gamma'$ be the sequence of actions returned by one execution of Algorithm \ref{alg:complete_dummies}. Then, $\gamma'$ is a sequence of actions that takes $A_a$ of $\Sys_3$ from state $\ell^0$ to $\ell^{acc}$.
\end{restatable}
\begin{proof}
    Let $\sigma$ be the accepting run of $\Sys_1$ that is given as an input to the algorithm. 
    Suppose that both $A$ in $\Sys_1$ and $A_a$ in $\Sys_3$ start in their initial states and, every time we read an action from $\alpha^\sigma$ in the algorithm (line 4), we take that action in the current state of $A$\footnote{Or from a state of $A$ that is reachable from the current one using only $\varepsilon$-transitions. \label{fn:eps_trans}}; we know that this is always possible since $\alpha^\sigma$ is by definition a sequence of actions\footnote{Possibly interleaved by some $\varepsilon$ actions. \label{fn:eps_act}} that takes $A$ from its initial state $\ell^0$ to its final state $\ell^{acc}$. Similarly, each time that an action is added by the algorithm to the sequence $\gamma'$, we take that action in the current state of $A_a$ in $\Sys_3$\footref{fn:eps_trans}, provided that there exists a transition with such an action; we show that there always is such a transition, and that the sequence of actions $\gamma'$ built by the algorithm\footref{fn:eps_act} will take $A_a$ from its initial state  $\ell^0$ to the state $\ell^{acc}$ (note that $\ell^{acc}$ is not the final state of $A_a$, by definition). We show by induction that, right before each iteration of the for loop, $A$ and $A_a$ can always be in the same state (the correspondence between states of $A$ and the states of $A_a$ is given by the definition of $A_a$); in particular, we show that before the $i$-th iteration of the for loop, they can always both be in state $\ell^{i-1}$, which is the state from which $A$ will take the action $\alpha^\sigma(i)$.
    For the base case, we are right before the first execution of the for loop. $A$ is in a state $\ell$ that is either the initial state $\ell^0$ or some state reachable with only $\varepsilon$-transitions from $\ell^0$. In both cases, by construction, $A_a$ can also be in the same state $\ell$. For the inductive step, we assume that $A$ and $A_a$ are both in the same state $\ell^{i-1}$ before executing the $i$-th iteration of the for loop (where $\ell^{i-1}$ is the state in which $A$ is ready to take the $\alpha^\sigma(i)$ action), and we show that at the end of the $i$-th iteration both $A$ and $A_a$ end up in state $\ell^i$. There are two main possibilities for the $i$-th iteration of the for loop:
    \begin{itemize}\itemsep=0.5ex
        \item $\alpha^\sigma(i)=!x$, so the \emph{if} at line 5 is entered. This means that $\ell^{i-1}$ in $A$ is a state that has an outgoing transition with the send action $!x$. By construction, since $A_a$ is also in $\ell^{i-1}$, it can take an unlimited number of $\sact{a}{c}{D}$ actions, followed by a $\sact{a}{b}{x}$ action. These are exactly the kind of actions added to $\gamma'$ by the \emph{if} that starts at line 5.
        \item $\alpha^\sigma(i)=?x$, so the \emph{if} at line 12 is entered. $A$ is then in a state $\ell^{i-1}$ that has an outgoing transition with the receive action $?x$. By construction, $A_a$ in state $\ell^{i-1}$ can take the $\sact{a}{c}{x}$ action, followed by a $\sact{a}{b}{D}$ action; after that, $A_a$ can take the consecutive pair of actions $\sact{a}{c}{D}$ and $\sact{a}{b}{D}$ any number of times. These are exactly the kind of actions added to $\gamma'$ by the \emph{if} that starts at line 12.
    \end{itemize}   
    In both cases, after taking action $\alpha^\sigma(i)$, $A$ gets to state $\ell^i$ (possibly using some additional $\varepsilon$-actions), ready for the next execution of the for loop; after taking the actions added by the algorithm to $\gamma'$, $A_a$ can also get to state $\ell^i$, by construction. After the last iteration of the for loop, $A$ and $A_a$ will both be in state $\ell^n = \ell^{acc}$. By construction, $A_a$ can take an unlimited number of $\sact{a}{c}{D}$ actions in this state, which are the only kind of actions that can be added by the algorithm during the final while loop (line 27). 
    \qed
\end{proof}

\begin{lemma}
    \label{lem:alg_FIFO_seq}
        Let $\gamma'$ be the sequence of actions returned by one execution of Algorithm \ref{alg:complete_dummies}. $\gamma'$ is a valid FIFO sequence.
\end{lemma}
\begin{proof}
    Each time that a $\sact{a}{b}{m}$ action gets added to $\gamma'$ by the algorithm, $m$ is enqueued, and each time a $\sact{a}{c}{m}$ action is added to $\gamma'$, $m$ is dequeued. Our claim directly follows (the behavior of a queue is naturally FIFO). 
\end{proof}

\begin{restatable}{lemma}{algTermination}
\label{lem:alg_termination}    
    Algorithm \ref{alg:complete_dummies} always terminates.
\end{restatable}
\begin{proof}
    The only ways in which the algorithm does not terminate are either
    \begin{enumerate*}[label={(\roman*)}]
        \item if it blocks when trying to dequeue a message $m$ that is not the first in \emph{Queue}, or
        \item if a while loop runs forever.
    \end{enumerate*} 
    We show that neither ever happens. Let us first focus on the specific case of line 14, when a message $x$ is dequeued.
    Each time the algorithm encounters a send action $!x$ in $\alpha^\sigma$, message $x$ is enqueued; each time it encounters a receive action $?x$, message $x$ is dequeued. There are no other occasions in which a normal message (i.e., not a dummy message $D$) is enqueued or dequeued. By definition, $\alpha^\sigma$ is a valid FIFO sequence for a single queue, so each time that the algorithm reads a receive action $?x$ and gets to line 14, message $x$ must be the first in the queue, unless there are some dummy messages $D$ before. We show that this is impossible. Suppose, by contradiction, that during the $i$-th iteration of the for loop, $\alpha^\sigma(i)=?b$ and the algorithm blocks at line 14, because there are some $D$ messages before $b$ in the queue; these $D$ messages must have been enqueued during previous iterations of the for loop. Note that $i>1$, since the first action in $\alpha^\sigma$ cannot be a receive action, and the algorithm gets to line 14 only when it reads a receive action.  Consider the previous $(i-1)$-th iteration of the algorithm, where $\alpha^\sigma(i-1)$ could either be a send or receive action: 
    \begin{itemize}\itemsep=0.5ex
        \item In the first case, we would have entered the while loop at line 6, which would have dequeued all the $D$ messages on the top of the queue, leaving $b$ as the first one when entering the $i$-th iteration of the for loop, therefore leading to a contradiction.
        \item In the second case, $\alpha^\sigma(i-1)$ is a receive action, so we would have entered the \emph{if} at line 17 (since $b$ is in the queue by hypothesis), and the while loop right after at line 18; this loop also dequeues all $D$ messages and puts them back in the queue, leaving $b$ as the first one when entering the $i$-th iteration of the for loop, leading again to a contradiction.
    \end{itemize}
    We showed that the \emph{dequeue} operation at line 14 never blocks. 
    
    Now, we consider the cases in which a $D$ message is dequeued (line 8, 20, and 29), and show that the algorithm never blocks. In all of these cases, we are in a while loop that is entered only if the message at the top of the queue is $D$, therefore the algorithm will never block.
    The last thing to show is that no while loop will run forever. To do this, we first show that, at any point of the algorithm, the number of $D$ messages in the queue is at most $n$. Note that a $D$ message can only be added to the queue at lines 16 and 22. In the case of line 22, a $D$ message is enqueued only after another $D$ message was dequeued (line 20), so the total number of $D$ messages in the queue does not change each time that an iteration of the while loop at line 18 is executed. Line 16 is therefore the only one that can effectively increase the number of $D$ messages in the queue, and can only be executed at most once per iteration of the for loop. The number of $D$ messages in the queue at any time can then be at most $n$ (in particular, it is finite). It follows directly that the while loops at line 6 and 27 will never run forever. We get to the while loop at line 18 only if there is at least one non-dummy message $x$ in the queue; since the number of $D$ in the queue is finite, the loop will run a finite number of time before encountering message $x$ at the top of the queue. 
    \qed
\end{proof}

\begin{lemma}
\label{lem:acc_single_to_three}
    If there is an accepting run $\sigma$ of $\Sys_1$, then there is an accepting run $\mu$ of $\Sys_3$.
\end{lemma}
\begin{proof}
    By Lemma \ref{lem:alg_termination}, Algorithm \ref{alg:complete_dummies} always terminates and returns $\gamma'$. By Lemma \ref{lem:alg_realiz_seq}, $\gamma'$ is a sequence of actions that takes $A_a$ from $\ell^0_a$ to $\ell^?_a$. Lemma \ref{lem:alg_FIFO_seq} shows that $\gamma'$ is a FIFO sequence, so we can finally use Lemma \ref{lem:fifo_realiz} to claim that there is an accepting run $\mu$ of $\Sys_3$ in which $!\alpha_a^\mu=\gamma'$. 
    \qed
\end{proof}

\begin{restatable}{lemma}{SendOrder} 
    \label{lem:send_order}
    Let $\mu$ be an accepting run of $\Sys_3$. In $!\alpha^\mu_a$, there is an equal number of messages sent to $b$ and to $c$. Moreover, in $!\alpha^\mu_a$, if  $x$ is the $i$-th message sent to $b$, and $y$ the $i$-th message sent to $c$, then $x=y$.
\end{restatable}
\begin{proof}
    By construction, in an accepting run of $\Sys_3$, $A_b$ must send messages in the same order and in the same number to the other two processes, in order not to block and to reach the state $\ell^?_b$, in which it is ready to start receiving messages. The same goes for $A_c$ (but not for $A_a$). Also, once $A_b$ gets to state $\ell^?_b$, it must receive messages from the other two processes in the same number (let it be $n$) and in the same order, so not to block and to reach the final state; this means that $a$ and $c$ must send exactly $n$ messages to $b$\footnote{By the definition of accepting run, $A_b$ has to receive all messages that were sent by the other two processes before moving to the final state, since we cannot have some messages sent to $b$ that are not received.}. The same kind of reasoning holds for $A_a$ and $A_c$, i.e., each process receives messages in the same order and in the same number from the other two processes in an accepting run of $\Sys_3$. We now show that the number of messages sent by $a$ to $b$ (let it be $n_1$) and by $a$ to $c$ (let it be $n_2$) is the same. Suppose, by contradiction, that $n_1 \neq n_2$ in an accepting run of $\Sys_3$. Based on the above, $n_1$ is also the number of messages sent by $c$ to $b$, and by $c$ to $a$; similarly, $n_2$ is the number of messages sent by $b$ to $c$, and by $b$ to $a$. We then have that $a$ receives $n_1$ messages from $c$, and $n_2$ messages from $b$. We said that we must have $n_1 = n_2$ in an accepting run of $\Sys_3$, hence the contradiction. The second part of the lemma essentially says that, for every accepting run of $\Sys_3$, the order in which $a$ sends messages to $b$ is the same as the order in which $a$ sends messages to $c$. By contradiction, suppose $a$ sends messages to $b$ following the order $X=m_1 \ldots m_k$, and messages to $c$ following another order $Y$, such that $X \neq Y$.    
    Based on the above, $X$ is also the order in which messages are sent by $c$ to $b$, and by $c$ to $a$; similarly, $Y$ is the order in which messages are sent by $b$ to $c$, and by $b$ to $a$. We then have that $a$ receives messages in order $X$ from $c$, and in order $Y$ from $b$. We said that we must have $X = Y$ in an accepting run of $\Sys_3$, hence the contradiction.
    \qed
\end{proof}
    
In order to make the following proofs more readable, we introduce some simplified terminology. Let $\mu$ be an accepting run of $\Sys_3$. In $!\alpha^\mu_a$, we will often refer to send actions addressed to $b$ as ``sends'', and to send actions addressed to $c$ as ``receipts'' (it follows from the way $\Sys_3$ was built from $\Sys_1$). Additionally, in $!\alpha^\mu_a$, we will refer to the $i$-th send action to $c$ as the matching receipt for the $i$-th send action to $b$ (which, in turn, will be referred to as the matching send for the $i$-th send action to $c$). For example, let $!\alpha^\mu_a= \sact{a}{c}{x}\;\sact{a}{b}{x}\;\sact{a}{b}{y}\;\sact{a}{b}{y}\;\sact{a}{b}{z}\;\sact{a}{c}{y}\;\sact{a}{c}{y}\;\sact{a}{c}{z}$ (note that it respects Lemma \ref{lem:send_order}): we will refer to the first $\sact{a}{c}{x}$ action as the receipt of the first $\sact{a}{b}{x}$ action, and similarly to $\sact{a}{c}{z}$ as the receipt of the only $\sact{a}{b}{z}$ action in $!\alpha^\mu_a$ (note that $\sact{a}{c}{z}$ is the 4th send action to $c$, and $\sact{a}{b}{z}$ is the 4th send action to $b$).

\begin{restatable}{lemma}{noRecBefSend}
\label{lem:no_rec_bef_send}
    Let $\mu$ be an accepting run of $\Sys_3$. For every message $x$, we cannot have more $\sact{a}{c}{x}$ actions than $\sact{a}{b}{x}$ actions in any prefix of $!\alpha^\mu_a$. 
\end{restatable}
\begin{proof}
    Using the above-mentioned simplified terminology, we could rephrase the lemma as: given an accepting run $\mu$ of $\Sys_3$, in $!\alpha^\mu_a$ there cannot be a receipt that appears before its send. By contradiction, suppose there is an accepting run $\mu$ of $\Sys_3$ in which a receipt appears before its matching send in $!\alpha^\mu_a$. Let us uniquely identify as $lastR_x$ the first such receipt in $!\alpha^\mu_a$, and as $lastS_x$ its matching send (we have $lastR_x \dashrightarrow^+ lastS_x$ in $!\alpha^\mu_a$). According to our reduction rules, we must send a dummy message right after $lastR_x$ in $!\alpha^\mu_a$. We will uniquely identify this action as $lastS_D$, and its receipt as $lastR_D$. There are two possibilities: either 
    \begin{enumerate*}[label={(\roman*)}]
        \item $lastS_D \dashrightarrow^+ lastR_D$ or
        \item $lastR_D \dashrightarrow^+ lastS_D$.
    \end{enumerate*}
   The first case leads to a contradiction, because we would have $lastS_D \dashrightarrow^+ lastS_x$ and $lastR_x \dashrightarrow^+ lastR_D$, which violates Lemma \ref{lem:send_order} (message $D$ is sent to $b$ before message $x$, but $D$ is sent to $c$ after $x$). We then consider the second scenario, in which $lastR_D \dashrightarrow^+ lastS_D$. According to the implementation of $\Sys_3$ (see reduction rules), a receipt of a dummy message, such as $lastR_D$, can only happen either 
   \begin{enumerate*}[label={(\roman*)}]
    \item before a send, or
    \item somewhere after a receipt (in any case, before the next non dummy-related action)\footnote{When not specified, actions do not refer to dummy messages. For example, ``can only happen before a send'' refers to the sending of a non-dummy message.}.
\end{enumerate*}
   The first case leads to a contradiction. Let $s$ be the above-mentioned send action and $r$ its receipt; we must have $s \dashrightarrow^+ r$, since $lastR_x$ was chosen as the first receipt that appears in $!\alpha^\mu_a$ before its send, but this violates again Lemma \ref{lem:send_order} (we would have $s \dashrightarrow^+ lastS_D$ and $lastR_D \dashrightarrow^+ r$). We then consider the second scenario, in which $lastR_D$ happens somewhere after a receipt of a message $y$, which we uniquely identify as $R_y$. Let $S_y$ be the matching send of $R_y$. For the same reason as before, $S_y \dashrightarrow^+ R_y$. According to our reduction rules, between $R_y$ and $lastR_D$ there could be an arbitrary large sequence of alternating $\sact{a}{b}{D}$ and $\sact{a}{c}{D}$ actions, where the last $\sact{a}{c}{D}$ is exactly $lastR_D$. In any case, if there are $k+1$ dummy messages sent between $R_y$ and $lastR_D$, there must be $k$ dummy messages received (excluding the end points). Let $S_D$ be any send of these dummy messages. Note that the matching receipt of $S_D$ (uniquely identified as $R_D$) cannot be neither
   \begin{enumerate*}[label={(\roman*)}]
    \item after $lastR_D$, nor
    \item before $R_y$,
\end{enumerate*}
   since both would again violate Lemma \ref{lem:send_order}: in the first case, we would have $S_D \dashrightarrow^+ lastS_D$ and $lastR_D \dashrightarrow^+ R_D$, whereas in the second case $S_y \dashrightarrow^+ S_D$ (since $S_y \dashrightarrow^+ R_y \dashrightarrow^+ S_D$) and $R_D \dashrightarrow^+ R_y$. This means that any of the $k$ sends of dummy messages between $R_y$ and $lastR_D$ must have its matching receipt also between $R_y$ and $lastR_D$ (end points excluded); this is impossible, since between $R_y$ and $lastR_D$ there are only $k$ dummy messages received and $k+1$ dummy messages sent, so at least one send will not have its matching receipt.
    \qed
\end{proof}

\begin{restatable}{lemma}{accMIffC}
\label{lem:acc_m_iff_c}
    If there is an accepting run $\mu$ of $\Sys_3$, then there is an accepting run $\sigma$ of $\Sys_1$.
\end{restatable}
\begin{proof}
    Given an accepting run $\mu$ of $\Sys_3$, Algorithm \ref{alg:run_C_to_M} always returns a sequence of actions $\alpha^\sigma$ for an accepting run $\sigma$ of $\Sys_1$.  The proof is very similar to that of Lemma \ref{lem:acc_single_to_three}, but much easier; therefore, we only describe the main intuition without dealing with most of the formalism. First, the algorithm removes all actions related to dummy messages from $!\alpha^\mu_p$, and creates the sequence $seq$; then, it returns the sequence $\gamma$, which is identical to $seq$, except that $\sact{a}{b}{x}$ and $\sact{a}{c}{x}$ actions are rewritten as $!x$ and $?x$, respectively. Let $l^0 \ldots l^?$ be the sequence of states traversed by $A_a$ while taking the actions in $!\alpha^\mu_a$, ignoring states in which the only outgoing transitions have a $\sact{a}{b}{D}$ or $\sact{a}{c}{D}$ action; more specifically, these are the intermediate states introduced by the first reduction rule, which do not have a one-to-one correspondence with states of $A$.  Note that, by Lemmas \ref{lem:send_order} and \ref{lem:no_rec_bef_send}, $seq$ and, therefore, $\gamma$, are FIFO sequences. By construction, it is now not difficult to see that the sequence of actions $\gamma$ takes $A$ from $l^0$ to $l^?$ where $l^?$ is its final state. After all, just by looking at how $A_a$ is constructed, it is clear that a sequence of send actions in $A_a$, when removing dummy messages actions and interpreting $\sact{a}{b}{x}$ and $\sact{a}{c}{x}$ as $!x$ and $?x$, also represents a valid sequence for $A$, as long as it is a valid FIFO sequence (otherwise some receive actions in $A$ might block trying to read a message that is not at the top of the queue).
    \qed
\end{proof}

\begin{algorithm}\caption{Let $\mu$ be an accepting run of $\Sys_3$, and $!\alpha^\mu_a$ be the sequence of $size$ send actions taken by $A_a$ in $\mu$. $!\alpha^\mu_a(i)$ denotes the $i$-th action of $!\alpha^\mu_a$.}
    \label{alg:run_C_to_M}
    \begin{multicols}{2}
    \begin{algorithmic}[1]
    \STATE $seq \gets !\alpha^\mu_a$  
    \FOR{$i$ from $1$ to $size$}
        \STATE $action$ $\gets$ $!\alpha^\mu_a(i)$
        \IF{$action = \sact{a}{b}{D}$ or $action = \sact{a}{c}{D}$}
            \STATE remove $action$ from $seq$
        \ENDIF
    \ENDFOR
    \STATE $n \gets length(seq)$
    \STATE $\gamma$ $\gets$ empty list
    \FOR{$i$ from $1$ to $n$}
        \STATE $action$ $\gets$ $seq(i)$
        \IF{$action = \sact{a}{b}{x}$}
            \STATE add $!x$ to $\gamma$
        \ELSIF{$action = \sact{a}{c}{x}$}
            \STATE add $?x$ to $\gamma$
        \ENDIF
    \ENDFOR
    \RETURN $\gamma$;
    \end{algorithmic}
    \end{multicols}
\end{algorithm}

\section{Weakly synchronous causally ordered MSCs}
\label{apx:ws_p2p_co}
We recall here the definition of causally ordered (CO) MSC, borrowed from \cite{DBLP:journals/pacmpl/GiustoFLL23}.

\begin{definition}[CO MSC]\label{def:co_msc}
	An MSC $\msc = (\Events,\procrel,\lhd,\lambda)$ is \emph{causally ordered} if, for any two send events $s$ and $s'$, such that $\lambda(s)\in \pqsAct{\plh}{q}$, $\lambda(s')\in \pqsAct{\plh}{q}$, and $s \happensbefore s'$:
	\begin{itemize}%\itemsep=0.5ex
		\item either $s,s' \in \Matched{\msc}$ and  $r \procrel^* r'$, with $r$ and $r'$   receive events such that $s \lhd r$ and $s' \lhd r'$.
		\item or $s' \in \Unm{\msc}$.
	\end{itemize}
\end{definition}

An MSC is weakly synchronous CO if it is a weakly synchronous MSC and a CO MSC.

\begin{theorem}
    An $MSC$ is weakly synchronous CO if and only if it is weakly synchronous $\pp$.
\end{theorem}
\begin{proof}
    ($\Leftarrow$) Let $M$ be a weakly synchronous CO MSC. $M$ is weakly synchronous and is also $\pp$, since each CO MSC is a $\pp$ MSC.

    \noindent($\Rightarrow$) Let $M$ be a weakly synchronous $\pp$ MSC. By contradiction, suppose it is not causally ordered, which means that there exist two send events $s$ and $s'$ addressed to the same process, such that $s {\le}_M s'$, and one of the following holds:
    \begin{itemize}
        \item $r' \procrel^+ r$, where $s \lhd r$ and $s' \lhd r'$. Note that $s$ and $s'$ cannot be executed by the same process, otherwise $M$ would not even be $\pp$. Since $s \le_M s'$, there is a 'chain' of events that causally links $s$ to $s'$. Note that, in this chain, there must exist a receive event $r''$ and a send event $s''$ such that $r'' \procrel^+ s''$ (otherwise $s$ and $s'$ could not be causally related). We now have a send event $s''$ that is executed after a receive event $r''$ by the same process. Note that $r''$ and $s''$ cannot be in two distinct phases of the weakly synchronous MSC $M$, since $s$ and $r$ (matching events) must be in the same phase, and we have that $s \le_M r'' \procrel^+ s'' \le_M s' \lhd r' \procrel^+ r$ (i.e., all these events between $s$ and $r$ must be part of the same phase).
        \item $s$ in unmatched, and $s' \lhd r'$. As before, note that $s$ and $s'$ cannot be executed by the same process, otherwise $M$ would not even be $\pp$. 
    \end{itemize}
\end{proof}
\end{document}